\documentclass[12pt]{article}

\usepackage[utf8]{inputenc}
\usepackage[margin=1in]{geometry}

\usepackage{amsmath,amssymb,amsthm}
\usepackage{graphicx}
\usepackage{authblk}

\usepackage{tikz}
\usepackage{tkz-graph}
\usetikzlibrary{shapes.geometric}
\usetikzlibrary{backgrounds}
\usetikzlibrary{arrows.meta}
\usetikzlibrary{positioning,calc}

\usepackage[colorlinks=true, allcolors=blue]{hyperref}
\usepackage{float}
\usepackage{caption}
\usepackage{subfig}
\usepackage{enumerate}

\newtheorem{theorem}{Theorem}[section]

\newtheorem{corollary}[theorem]{Corollary}
\newtheorem{proposition}[theorem]{Proposition}

\newtheorem{setting}{Setting}[section]
\newtheorem{definition}[theorem]{Definition}
\newtheorem{example}[theorem]{Example}

\numberwithin{equation}{section}

\providecommand{\keywords}[1]{%
  \small 
  \textbf{\textit{Keywords---}} #1
}

\newcommand{\indep}{\perp \!\!\! \perp}
\newcommand{\twoset}[2]{%
  \mathop{#2}\limits^{\vbox to .3ex{%
  \kern -0.7ex\hbox{$#1$}\vss}}}

\makeatletter
\newcommand\sfcodefork{%
  \ifnum\the\spacefactor=1000 \expandafter\@firstoftwo\else\expandafter\@secondoftwo\fi
}%

\makeatother

\title{Stable Blanket with Hidden Variables and Cycles}

\author[1]{Hanqing Xiang\thanks{This work is based on research conducted during the author's MSc studies at the University of Copenhagen under the supervision of Niklas Pfister.}}
\affil[1]{Department of Mathematics, KTH Royal Institute of Technology, Sweden}

\date{}

\usepackage[T1]{fontenc}
\usepackage[utf8]{inputenc}
\begin{document}

\nonfrenchspacing

\maketitle

\begin{abstract}

Stabilized regression aims to identify a set of predictors whose conditional relationship with a response variable remains invariant across different environments. Existing graphical characterizations of the stable blanket are mainly developed for structural causal models (SCMs) without hidden variables or causal cycles. However, latent variables and feedback relationships naturally arise in many applications, and they can change both the Markov blanket and the set of predictors that remain stable under interventions. This paper studies stable blankets in graphical causal models with hidden variables, causal cycles, and both features simultaneously. For models with hidden variables, we use acyclic directed mixed graphs (ADMGs) and $m$-separation to characterize the Markov blanket and to construct intervention-stable predictor sets. We introduce the notion of an intervened sub-district and use it to describe how interventions may affect districts connected to the response. For models with cycles, we work with directed graphs (DGs) and directed mixed graphs (DMGs) together with $\sigma$-separation, treating strongly connected components (SCCs) as the basic graphical units. We then combine these ideas to analyze models with both hidden variables and cycles. The main results give graphical characterizations of Markov blankets, stable frontiers, and stable blankets in these generalized settings. In particular, we identify conditions under which the response is conditionally independent of intervention variables given a suitable predictor set, and we describe when such sets are minimal or unique. These results extend the graphical interpretation of stabilized regression beyond acyclic fully observed models.

\end{abstract}


\section{Introduction}

Statistical methods such as regression are widely used to describe relationships between a response variable and a set of predictors. In multi-environment settings, however, a regression relationship that holds in one environment may fail to generalize to another. This issue is closely related to the distinction between statistical association and the underlying causal data-generating process. Structural causal models (SCMs) provide a framework for representing causal relationships and for interpreting regression across different environments \cite{Judea09,Spirtes00}. Within this framework, \cite{Niklas21} proposed stabilized regression, a methodology for identifying predictive relationships that remain stable under changes of environment. The central graphical object in stabilized regression is the stable blanket, a set of predictors whose conditional relationship with the response can extrapolate to unseen environments. Figure~\ref{fig:M1} illustrates the Markov blanket and the stable blanket in a model without hidden variables or causal cycles. The graphical framework in \cite{Niklas21}, however, does not cover hidden variables or causal cycles. This paper studies more general settings in which latent variables and feedback loops may be present, and aims to characterize stable predictor sets that are both informative for the response and invariant across unseen interventional environments.

    

\begin{figure}[ht]
\centering
\begin{tikzpicture}[
    >=Stealth,
    thick,
    scale=1.3,
    every node/.style={font=\normalsize},
    obsnode/.style={
        circle,
        draw=black,
        fill=white,
        minimum size=9mm,
        inner sep=0pt
    },
    intnode/.style={
        rectangle,
        draw=black,
        fill=white,
        minimum width=9mm,
        minimum height=9mm,
        inner sep=0pt
    },
    ynode/.style={
        double,
        circle,
        draw=black,
        fill=white,
        minimum size=10mm,
        inner sep=0pt,
        line width=0.8pt
    }
]

    \node[intnode] (i1) at (0,5) {$I^1$};
    \node[obsnode] (x1) at (2,5) {$X^1$};
    \node[obsnode] (x2) at (4,5) {$X^2$};
    \node[ynode]   (y)  at (2,3.5) {$Y$};
    \node[obsnode] (x6) at (0.5,3.5) {$X^6$};
    \node[obsnode] (x5) at (1.2,2.2) {$X^5$};
    \node[obsnode] (x7) at (2,1) {$X^7$};
    \node[obsnode] (x3) at (3,2.2) {$X^3$};
    \node[obsnode] (x4) at (3.6,3.5) {$X^4$};
    \node[intnode] (i2) at (4.6,2.2) {$I^2$};

    \draw[->] (i1) -- (x1);
    \draw[->] (x1) -- (y);
    \draw[->] (x2) -- (y);
    \draw[->] (x6) -- (x5);
    \draw[->] (y) -- (x5);
    \draw[->] (y) -- (x3);
    \draw[->] (x4) -- (x3);
    \draw[->] (i2) -- (x3);
    \draw[->] (x5) -- (x7);
    \draw[->] (x3) -- (x7);

\end{tikzpicture}
\caption{A DAG without hidden variables. The intervention nodes are $I^1$ and $I^2$. The Markov blanket of $Y$ is $\{X^1,X^2,X^3,X^4,X^5,X^6\}$, while the stable blanket of $Y$ is $\{X^1,X^2,X^5,X^6\}$.}
\label{fig:M1}
\end{figure}
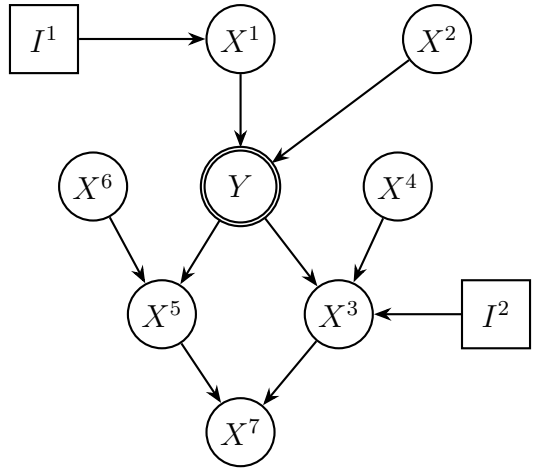

Hidden variables and causal cycles arise naturally in many applications, including systems biology, neuroscience, econometrics, and cognitive science. In complex biological systems, for example, it is often impossible to observe all relevant factors, and unobserved variables may induce additional dependencies among the observed variables. In econometric models, reciprocal causal relationships such as those between price and supply may lead to cycles in the corresponding graph. In such settings, the acyclic fully observed framework may no longer identify the desired stable blanket. Figure~\ref{fig:M2} shows two simple examples. The left graph illustrates a model with a hidden variable, while the right graph illustrates a model with a causal cycle. These examples motivate the use of richer graphical models to describe the relationship between the response and the predictors.

\begin{figure}[ht]
\centering

\begin{minipage}[c]{0.44\textwidth}
\centering
\begin{tikzpicture}[
    baseline={(current bounding box.center)},
    >=Stealth,
    thick,
    scale=1.15,
    every node/.style={font=\normalsize},
    obsnode/.style={
        circle,
        draw=black,
        fill=white,
        minimum size=9mm,
        inner sep=0pt
    },
    intnode/.style={
        rectangle,
        draw=black,
        fill=white,
        minimum width=9mm,
        minimum height=9mm,
        inner sep=0pt
    },
    hiddennode/.style={
        circle,
        draw=black,
        dashed,
        fill=gray!15,
        minimum size=9mm,
        inner sep=0pt
    },
    ynode/.style={
        double,
        circle,
        draw=black,
        fill=white,
        minimum size=10mm,
        inner sep=0pt,
        line width=0.8pt
    }
]
    \node[obsnode]    (x1) at (0,4) {$X^1$};
    \node[hiddennode] (h)  at (2,4) {$H$};
    \node[ynode]      (y1) at (1,2.5) {$Y$};
    \node[obsnode]    (x2) at (3,2.5) {$X^2$};
    \node[intnode]    (i1) at (-2,2.5) {$I^1$};
    \node[obsnode]    (x3) at (-0.5,2.5) {$X^3$};

    \draw[->] (i1) -- (x3);
    \draw[->] (x1) -- (y1);
    \draw[->] (y1) -- (x3);
    \draw[->] (h) -- (y1);
    \draw[->] (h) -- (x2);
\end{tikzpicture}
\end{minipage}
\hfill
\begin{minipage}[c]{0.44\textwidth}
\centering
\begin{tikzpicture}[
    baseline={(current bounding box.center)},
    >=Stealth,
    thick,
    scale=1.15,
    every node/.style={font=\normalsize},
    obsnode/.style={
        circle,
        draw=black,
        fill=white,
        minimum size=9mm,
        inner sep=0pt
    },
    ynode/.style={
        double,
        circle,
        draw=black,
        fill=white,
        minimum size=10mm,
        inner sep=0pt,
        line width=0.8pt
    }
]
    \node[ynode]   (y2) at (0,2) {$Y$};
    \node[obsnode] (u1) at (2,4) {$X^1$};
    \node[obsnode] (u2) at (4,2) {$X^2$};
    \node[obsnode] (u3) at (3,0) {$X^3$};
    \node[obsnode] (u4) at (1,0) {$X^4$};

    \draw[->] (y2) -- (u1);
    \draw[->] (u1) -- (u2);
    \draw[->] (u2) -- (u3);
    \draw[->] (u3) -- (u4);
    \draw[->] (u4) -- (y2);
\end{tikzpicture}
\end{minipage}

\caption{Two motivating examples. Left: a model with a hidden variable $H$ and an intervention $I^1$. The predictor $X^2$ is relevant for the stable blanket of $Y$ because $X^2$ and $Y$ share a hidden common cause. Right: a model with a causal cycle. Although $X^2$ is not directly adjacent to $Y$, it may still be one of the most informative predictors for $Y$.}
\label{fig:M2}
\end{figure}
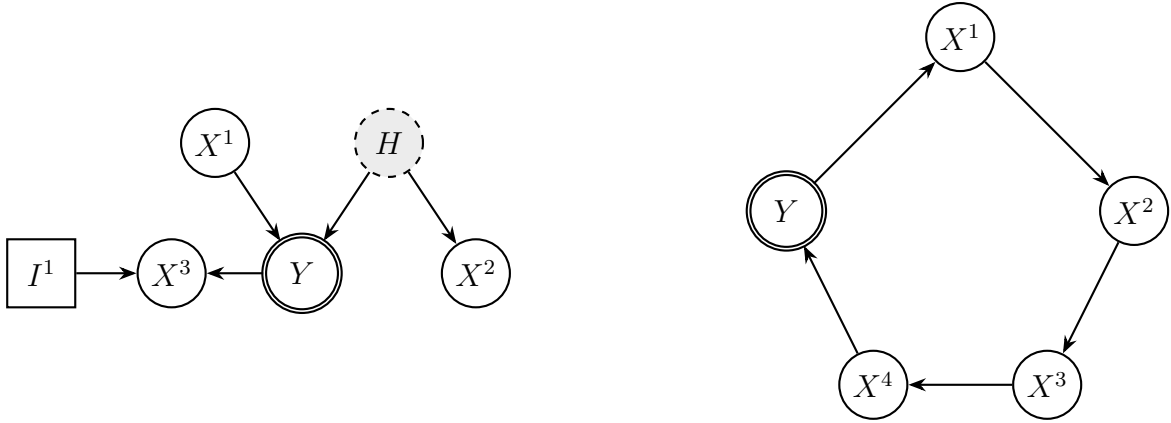

Several approaches have been developed for graphical models with hidden variables. For example, \cite{Friedman13} incorporated hidden variables into graphical models and considered recovery of the underlying model by estimating latent variables. The Markov blanket has also been studied in models with hidden variables and cycles, including in the context of cognitive science, where \cite{Bruineberg21} distinguish the Markov blanket from the so-called realistic blanket using Bayesian inference. In another related direction, \cite{Triantafillou21} considered hidden variables in settings with a single treatment and used observational data to improve feature selection and effect estimation. In contrast, the present paper allows multiple interventions and focuses on graphical characterizations of stable predictor sets. Two graphical formalisms are particularly relevant: acyclic directed mixed graphs (ADMGs) \cite{Richardson03} and maximal ancestral graphs (MAGs) \cite{Richardson02}. MAGs encode ancestral relationships, whereas ADMGs represent hidden-variable-induced dependencies through bidirected edges obtained by latent projection. We work with ADMGs because their district structure provides a convenient way to characterize Markov blankets and stable predictor sets. In this setting, the Markov blanket can be decomposed into several graphical components, but not every graph admits an environment-independent predictor set. Such stable sets are useful because regression functions based on them can generalize to unseen environments, and because they can provide interpretations closer to the stable causal relationship. At the same time, they may contain less information about the response than the full Markov blanket.

We next consider models with reciprocal causal relationships, where the graphical model contains cycles. In this case, the usual properties of acyclic graphical models no longer apply. We therefore study the relevant Markov property using $\sigma$-separation in directed graphs (DGs) with cycles. In cyclic models, variables belonging to the same strongly connected component (SCC) play a similar graphical role. We use these components as basic units and obtain characterizations of Markov blankets and stable blankets that parallel the acyclic case studied in \cite{Niklas21}.

To the best of our knowledge, stable blankets have not previously been graphically characterized in models with hidden variables and causal cycles simultaneously. We study this problem using directed mixed graphs (DMGs). This framework includes both hidden-variable models and cyclic models as special cases, but the results from the two separate settings cannot be combined directly. Additional conditions are needed to ensure that the relevant separation and stability properties continue to hold. We adapt the construction of stable predictor sets from the hidden-variable case and the cyclic case to this more general setting. Although DMGs with cycles may be less directly interpretable causally, they can be related to the two simpler settings through decompositions of cycles \cite{Bongers21}. Under the corresponding Markov property, we show that the conditional relationship between the response and the stable predictor set remains invariant under interventions.

The previous discussion assumes that interventions do not act on the district of the response. We also consider a weaker assumption under which stable blankets may still exist even when interventions act on the district of the response. This allows us to classify predictors into different graphical types and extends the applicability of the framework.

\paragraph{Contribution}
The contributions of this paper are as follows. First, in models with hidden variables, we introduce the notion of an intervened sub-district and use it to construct predictor sets on which the conditional distribution of the response is unaffected by interventions. We also characterize the corresponding smallest set and provide conditions for uniqueness. Second, in models with cycles, we characterize the unique smallest predictor set that is both informative for the response and invariant under interventions. Third, in models with both hidden variables and cycles, we establish graphical results for predictor sets that remain stable under interventions. Finally, we identify a necessary condition for the existence of a set that separates the response from the intervention variables and prove a corresponding theoretical result.

\paragraph{Outline}
The paper is organized as follows. Section~2 recalls the necessary preliminaries on graphical notation and separation criteria. Section~3 studies the Markov blanket and stable blanket in graphical models with hidden variables. Section~4 analyzes the Markov blanket and stable blanket in graphical models with causal cycles. Section~5 considers graphical models with both hidden variables and causal cycles, characterizes the corresponding Markov blanket and stable blanket, and relates them to the previous definitions. Section~6 discusses both blankets under weaker assumptions on interventions.

\section{Preliminaries}
\label{sec:preliminaries}

This section recalls the graphical notation and separation criteria used throughout the paper. We only introduce the notions that are needed for the main results. Standard graph-theoretic concepts such as paths, parents, children, ancestors, descendants, spouses, districts, colliders, and induced subgraphs are used in their usual sense; see, for example, \cite{Lauritzen96,Bongers21}.

\subsection{Directed and Mixed Graphs}

A directed graph (DG) is a pair $\mathcal{G}=(\mathcal{V},\mathcal{E})$, where $\mathcal{V}$ is a set of nodes and $\mathcal{E}$ is a set of directed edges. We write $i\to j$ if there is a directed edge from $i$ to $j$. A directed mixed graph (DMG) is a triple $\mathcal{G}=(\mathcal{V},\mathcal{E},\mathcal{B})$, where $\mathcal{B}$ is a set of bidirected edges. We write $i\leftrightarrow j$ if there is a bidirected edge between $i$ and $j$. Thus a DG is a special case of a DMG with $\mathcal{B}=\emptyset$.

A directed acyclic graph (DAG) is a DG with no directed cycles. An acyclic directed mixed graph (ADMG) is a DMG whose directed part has no directed cycles. For a DMG $\mathcal{G}$ and a node $i\in\mathcal{V}$, we write $pa(i)$, $ch(i)$, $an(i)$, $de(i)$, $sp(i)$, and $dis(i)$ for the parents, children, ancestors, descendants, spouses, and district of $i$, respectively. For a set $S\subseteq\mathcal{V}$, these notions are extended by taking unions; for example,
\[
pa(S)=\bigcup_{i\in S}pa(i).
\]
The induced subgraph of $\mathcal{G}$ over $S$ is denoted by $\mathcal{G}_S$.

When cycles are present, strongly connected components (SCCs) are needed. The SCC of a node $i$, denoted by $scc(i)$, is the maximal set $S\subseteq\mathcal{V}$ such that $i\in S$ and every two nodes in $S$ are connected by directed paths in both directions \cite{Mooij17}. Equivalently,
\[
scc(i)=(an(i)\cap de(i))\cup \{i\}.
\]

\subsection{Separation Criteria}

In DAGs, conditional independence is read off using $d$-separation. Since the fully observed acyclic case is only used as the baseline setting in this paper, we do not recall the full definition here; see \cite{Lauritzen96}. We write
\[
X \indep_d Y \mid Z
\]
when the node sets $X$ and $Y$ are $d$-separated by $Z$.

For ADMGs, we use $m$-separation \cite{Richardson03}. A path is $m$-blocked by a node set $Z$ if either it contains a non-collider in $Z$, or it contains a collider that is not in $an(Z)$. Two node sets $X,Y\subseteq\mathcal{V}$ are $m$-separated by $Z$, written
\[
X \indep_m Y \mid Z,
\]
if every path between a node in $X$ and a node in $Y$ is $m$-blocked by $Z$.

For DMGs with cycles, we use $\sigma$-separation \cite{Mooij17}. A path is $\sigma$-blocked by a node set $Z$ if one of the following conditions holds:
\begin{enumerate}[(i)]
    \item one of the endpoints of the path belongs to $Z$;
    \item the path contains a collider $j$ such that $j\notin Z$ and $j\notin an(Z)$;
    \item the path contains a non-endpoint non-collider $i\in Z$ whose adjacent edge on the path points towards a node outside $scc(i)$.
\end{enumerate}
Two node sets $X,Y\subseteq\mathcal{V}$ are $\sigma$-separated by $Z$, written
\[
X \indep_\sigma Y \mid Z,
\]
if every path between a node in $X$ and a node in $Y$ is $\sigma$-blocked by $Z$.

\subsection{Latent Projection}

Latent projection provides a graphical representation of the observed part of a model with hidden variables. Let $\mathcal{G}=(\mathcal{V},\mathcal{E},\mathcal{B})$ be a DMG with node set
\[
\mathcal{V}=O\mathop{\dot{\bigcup}}H,
\]
where $O$ is the set of observed nodes and $H$ is the set of hidden nodes. The latent projection of $\mathcal{G}$ onto $O$ is the DMG $\mathcal{G}^O$ over $O$ defined as follows \cite{Bongers21,Verma93}:
\begin{enumerate}[(i)]
    \item $i\to j$ is in $\mathcal{G}^O$ if and only if there exists a directed path
    \[
    i\to h^1\to \cdots \to h^n \to j
    \]
    in $\mathcal{G}$ with $n\geq 0$ and $h^1,\ldots,h^n\in H$;

    \item $i\leftrightarrow j$ is in $\mathcal{G}^O$ if and only if there is a path through hidden nodes with arrowheads into both $i$ and $j$. More explicitly, this holds if either
    \[
    i \leftarrow h^1 \leftarrow \cdots \leftarrow h^n
    \leftrightarrow
    \tilde h^m \rightarrow \cdots \rightarrow \tilde h^1 \rightarrow j
    \]
    for some $n,m\geq 0$, or if
    \[
    i \leftarrow h^1 \leftarrow \cdots \leftarrow h^n
    =
    \tilde h^m \rightarrow \cdots \rightarrow \tilde h^1 \rightarrow j
    \]
    for some $n,m\geq 1$, where all intermediate nodes are in $H$.
\end{enumerate}
In particular, when the original graph is directed, bidirected edges in the latent projection represent dependencies induced by hidden common causes.

\subsection{Markov Properties and Structural Causal Models}

Let $\mathbf{X}=(X^i)_{i\in\mathcal{V}}$ be random variables indexed by the nodes of a graph $\mathcal{G}$. A distribution $P(\mathbf{X})$ satisfies the Markov property with respect to a DAG $\mathcal{G}$ if
\[
A \indep_d B \mid C
\quad \Longrightarrow \quad
X_A \indep X_B \mid X_C
\]
for all disjoint node sets $A,B,C\subseteq\mathcal{V}$ \cite{Lauritzen96}. Similarly, for an ADMG we use $m$-separation,
\[
A \indep_m B \mid C
\quad \Longrightarrow \quad
X_A \indep X_B \mid X_C,
\]
and for a DMG with cycles we use $\sigma$-separation,
\[
A \indep_\sigma B \mid C
\quad \Longrightarrow \quad
X_A \indep X_B \mid X_C.
\]
We then write
\[
P(\mathbf{X})\mapsto(\mathcal{G},d\text{-separation}),\qquad
P(\mathbf{X})\mapsto(\mathcal{G},m\text{-separation}),\qquad
P(\mathbf{X})\mapsto(\mathcal{G},\sigma\text{-separation}),
\]
respectively.

Structural causal models (SCMs) provide an important class of models for which such graphical Markov properties can be studied. An SCM over random variables $X=(X^1,\ldots,X^p)$ is a collection of assignments
\[
X^i=f^i(pa(X^i),\varepsilon^i),
\qquad i=1,\ldots,p,
\]
where the noise variables $\varepsilon^1,\ldots,\varepsilon^p$ are jointly independent \cite{Judea09}. An SCM with hidden variables additionally contains unobserved variables $H^1,\ldots,H^q$ with structural assignments of the same form. Such a model induces a DG by adding an edge from each variable on the right-hand side of an assignment to the corresponding variable on the left-hand side.

When hidden variables are marginalized out, latent projection gives a mixed graph over the observed variables. Under suitable solvability assumptions, marginalization of SCMs with hidden variables is compatible with latent projection \cite{Bongers21}. The conditions under which $d$-separation, $m$-separation, and $\sigma$-separation imply conditional independence are discussed in \cite{Lauritzen96}, \cite{Richardson03}, and \cite{Mooij17}, respectively.

\subsection{Baseline Setting: SCMs without Hidden Variables and Cycles}

We recall the baseline setting of stabilized regression from \cite{Niklas21}.

\begin{setting}
\label{Setting 1}
Let
\[
X\in \mathcal{X}=\mathcal{X}^1\times\cdots\times\mathcal{X}^p
\]
be observable predictors, let $Y\in\mathbb{R}$ be a response variable, and let
\[
I=(I^1,\ldots,I^m)\in \mathcal{I}
=
\mathcal{I}^1\times\cdots\times\mathcal{I}^m
\]
be intervention variables encoding changes of environment. Assume an SCM $\mathcal{S}$ over $(I,X,Y)$ such that the induced graph $\mathcal{G}(\mathcal{S})$ is a DAG. The intervention variables are source nodes in $\mathcal{G}(\mathcal{S})$ and do not appear in the assignment of $Y$. Each intervention environment $e$ corresponds to an interventional SCM $\mathcal{S}_e$ over $(I_e,X_e,Y_e)$ with the same graph,
\[
\mathcal{G}(\mathcal{S}_e)=\mathcal{G}(\mathcal{S}).
\]
Let $\varepsilon^{tot}$ be a finite set of observed environments. Assume that the distribution of $(I_e,X_e,Y_e)$ is absolutely continuous with respect to a factorizing product measure.
\end{setting}

Under Setting~\ref{Setting 1}, define
\[
N^{-int}
=
\{1,\ldots,p\}
\setminus
\left\{
j\in\{1,\ldots,p\}
:
\exists k\in ch^{int}(Y)
\text{ such that }
j\in de(X^k)
\text{ or }
j=k
\right\},
\]
where $ch^{int}(Y)$ denotes the children of $Y$ that are directly intervened on. The stable blanket of $Y$ is the smallest set $S\subseteq N^{-int}$ such that
\[
X^j \indep_d Y \mid X^S,
\qquad
\forall j\in N^{-int}\setminus S.
\]
This set can be interpreted as the smallest predictor set, among the predictors not affected by the relevant interventions, that retains the information about $Y$ needed for stable prediction. Equivalently, the response $Y$ is independent of the intervention variables given the stable blanket. Figure~\ref{fig:M1} gives a graphical illustration of this baseline case.

\section{SCMs with Hidden Variables}

We first extend the graphical characterization of Markov blankets and stable blankets to models with hidden variables but without causal cycles. Let $Y$ denote the response variable and let $X$ denote the observed predictors. As in the fully observed acyclic case, the goal is to identify a small predictor set that retains the information about $Y$ needed for prediction. In the presence of interventions, we further require this predictor set to have a conditional relationship with $Y$ that remains invariant across environments.

Hidden variables create an additional difficulty. A hidden variable may lie in the Markov blanket or stable blanket of $Y$ in the underlying fully observed graph. Simply removing hidden variables can therefore lose information about $Y$ and may also destroy stability under interventions. To compensate for unobserved variables, additional observed predictors may need to be included, even when they are not directly adjacent to $Y$ in the original graph.

An SCM with hidden variables induces a DAG over all observed and hidden variables, provided that there are no causal cycles. After latent projection onto the observed variables and intervention variables, we obtain an ADMG. We use this ADMG to read off conditional independence relations among the observed variables. This requires that the projected distribution satisfies the Markov property with respect to the ADMG and $m$-separation. Under these assumptions, the graph can be used to determine which predictors are affected by interventions and which predictor sets remain stable. In contrast to the fully observed case, the assumption that interventions do not act directly on $Y$ is not sufficient; additional assumptions on the district of $Y$ are needed.

\begin{setting}
\label{Setting 2}
Let $X \in \mathcal{X} = \mathcal{X}^1 \times \cdots \times \mathcal{X}^p$ be observable predictors, $H \in \mathcal{H} = \mathcal{H}^1 \times \cdots \times \mathcal{H}^q$ be hidden variables, $Y \in \mathbb{R}$ be a response variable and $I = (I^1,\cdots,I^m) \in \mathcal{I} = \mathcal{I}^1 \times \cdots \times \mathcal{I}^m$ be intervention variables which are used to formalize the interventions and act on observed random variables. Assume there exists an SCM with hidden variables $\mathcal{S}$ over $(I,X,H,Y)$ such that $\mathcal{G}(\mathcal{S})$ is a DAG. Each intervention environment $e$ corresponds to an interventional SCM with hidden variables $\mathcal{S}_e$ over $(I_e,X_e,H_e,Y_e)$ where the $\mathcal{G}(\mathcal{S}_e)$ is a DAG and fixed (i.e., $\mathcal{G}(\mathcal{S}_e) = \mathcal{G}(\mathcal{S})$). Thus the ADMG $\mathcal{G}(\mathcal{S}_e)^{(I_e, X_e, Y_e)}$ generated by latent projection on $\mathcal{G}(\mathcal{S}_e)$ over $(I_e, X_e, Y_e)$ is fixed. Moreover, $I^l,\ l = 1, \cdots, m$ are source nodes and have no edges on the $dis(Y)$. Lastly, assume these SCMs with hidden variables are ancestrally uniquely solvable with respect to $H$.
\end{setting}

The last assumption of these SCMs with hidden variables can guarantee the distributions $P^{(I_e, X_e, Y_e)}$ satisfy the Markov property relative to $\mathcal{G}(\mathcal{S}_e)^{(I_e, X_e, Y_e)}$ based on $m$-separation.

\subsection{Markov Blanket in ADMGs}

In an ADMG $\mathcal{G}$, there are at most two edges between two nodes (otherwise, there will be a cycle). Let $i, j$ be two adjacent nodes in $\mathcal{G}$, we can use $i \twoset{-}{-} j$ to represent the connection between $i$ and $j$, where each “$-$” is a possible edge between the two nodes.

We say a path $P = (i^0, e^1, i^1, e^2, \cdots, e^n, i^n)$ has the shape
$$
i_0 \twoset{-}{-} i_1 \twoset{-}{-} \cdots \twoset{-}{-} i_n
$$
if $e^k$ is chosen from $\twoset{-}{-}$ between $i_{k-1}$ and $i_k$, $k = 1,\cdots,n$.

\begin{proposition}
\label{Block}
Given a path $P$ with shape $i_0 \twoset{-}{-} i_1 \twoset{-}{-} \cdots \twoset{-}{-} i_n$, if there are three successive nodes $(i_{k-1} ,i_k, i_{k+1})$ such that one of the connection relationships in Table 1 happens,

\begin{table}[ht]
\centering
\renewcommand{\arraystretch}{1.5}
    \begin{tabular}{l|l|l}
    \hline

    $i_{k-1} \rightarrow i_k \leftarrow i_{k+1}$ & $i_{k-1} \leftrightarrow i_k \leftarrow i_{k+1}$ & $i_{k-1} \rightarrow i_k \leftrightarrow i_{k+1}$\\  \hline

    $i_{k-1} \leftrightarrow i_k \leftrightarrow i_{k+1}$ & $i_{k-1} \twoset{\rightarrow}{\leftrightarrow} i_k \leftarrow i_{k+1}$ & $i_{k-1} \twoset{\rightarrow}{\leftrightarrow} i_k \leftrightarrow i_{k+1}$ \\ \hline

    $i_{k-1} \rightarrow i_k \twoset{\leftarrow}{\leftrightarrow} i_{k+1}$   & $i_{k-1} \leftrightarrow i_k \twoset{\leftarrow}{\leftrightarrow} i_{k+1}$ & $i_{k-1} \twoset{\rightarrow}{\leftrightarrow} i_k \twoset{\leftarrow}{\leftrightarrow} i_{k+1}$ \\ \hline
    \end{tabular}
    \label{Table 1}
    \caption{Cases where $i_k$ must be a collider}
\end{table}

\noindent then $P$ is $m$-blocked by any set $S$ which satisfies $i_k \notin an(S) \cup S$.
\end{proposition}

\begin{proof}
In every situation of $(i_{k-1} ,i_k, i_{k+1})$, no matter what $e^k$ and $e^{k+1}$ are, $i^k$ is a collider on $P$. Since $i_k \notin an(S) \cup S$, $P$ is $m$-blocked by $S$.
\end{proof}

Proposition \ref{Block} shows that all paths that share the same shape can be $m$-blocked simultaneously, as long as there are specific sub-shapes within that shape. Indeed, when there is a common non-collider on all paths with the same shape, these paths can also be $m$-blocked by a set that contains this non-collider. These results can simplify the kinds of paths to the response $Y$.

\begin{definition}[Markov Blanket]
In the graphical models of Setting \ref{Setting 2}, the Markov blanket of $Y$ is defined as the smallest set $S \subseteq \{1,\cdots,p\}$ which satisfies
\begin{align*}
    \forall j \in \{1,\cdots,p\}\backslash S,\ X^j \indep_m Y | X^S.
\end{align*}
\end{definition}

In a DAG, the Markov blanket of one variable consists of its parents, its children, and the parents of its children. This characterization allows intuitively reading off the smallest and simultaneously the most informative set of predictors from the graph. It also inspires us to classify the predictors in terms of graphical relationships between predictors and Y in each class. Similarly, we can construct a decomposition of the Markov blanket in the ADMG. 

\begin{proposition} (\cite{Richardson03})
\label{MB hidden}
In the graphical models of Setting \ref{Setting 2}, the Markov blanket of $Y$ is
\begin{align*}
    MB(Y) = pa(dis(Y)) \cup (dis(Y) \backslash Y) \cup dis(ch(Y)) \cup pa(dis(ch(Y))) 
\end{align*}
\end{proposition}

\begin{proof}
Let $S = pa(dis(Y)) \cup (dis(Y) \backslash Y) \cup dis(ch(Y)) \cup pa(dis(ch(Y)))$. We first show that $\forall j \in \{1,\cdots,p\} \backslash S$, it holds that $X^j \indep_m Y | X^S$. To this end, it suffices to prove that every path between $X^j$ and $Y$ is $m$-blocked by $X^S$. Let $P$ be a path connecting $X^j$ and $Y$. Then $P$ has one of the four possible shapes:
\begin{enumerate}[(i)]
    \item $X^j \cdots X^{k} \rightarrow X^{i_n} \twoset{-}{\leftrightarrow} \cdots \twoset{-}{\leftrightarrow} X^{i_1} \twoset{-}{\leftrightarrow} Y$ 
    
    where $k \notin sp(X^{i_n})$ and $n \geq 0$. Then $X^k$ is a non-collider on $P$ and $k \in pa(dis(Y))$. So $P$ is $m$-blocked by $X^S$.
    
    \item $X^j \cdots X^{k} \leftarrow X^{i_n} \twoset{-}{\leftrightarrow} \cdots \twoset{-}{\leftrightarrow} X^{i_1} \twoset{-}{\leftrightarrow} Y$ 
    
    where $k \notin sp(X^{i_n})$ and $n \geq 1$. Here it is possible that $X^j = X^{k}$. Then $X^{i_n}$ is a non-collider on $P$ and $i_n \in dis(Y) \backslash Y$. Thus $P$ is $m$-blocked by $X^S$.
    
    \item $X^j \cdots X^{k} \rightarrow X^{i_n} \twoset{-}{\leftrightarrow} \cdots \twoset{-}{\leftrightarrow} X^{i_1} \leftarrow Y$ 
    
    where $k \notin sp(X^{i_n})$, $i_1 \notin sp(Y)$, and $n \geq 1$. Then $X^k$ is a non-collider on $P$ and $k \in pa(dis(ch(Y)))$. So $P$ is $m$-blocked by $X^S$.
    
    \item $X^j \cdots X^{k} \leftarrow X^{i_n} \twoset{-}{\leftrightarrow} \cdots \twoset{-}{\leftrightarrow} X^{i_1} \leftarrow Y$ 
    
    where $k \notin sp(X^{i_n})$, $i_1 \notin sp(Y)$, and $n \geq 1$. Here it is possible that $X^j = X^{k}$. Then $X^{i_n}$ is a non-collider on $P$ and $i_n \in dis(ch(Y))$. Thus $P$ is $m$-blocked by $X^S$.
\end{enumerate}

Therefore, we have shown that every path between $X^j$ and $Y$ is $m$-blocked by $X^S$, which means that $X^j \indep_m Y | X^S,\ \forall j \in \{1,\cdots,p\} \backslash S$. Then we need to prove that $S$ is the smallest subset $S^\prime$ such that $\forall j \in \{1,\cdots,p\} \backslash S^\prime,\ X^j \indep_m Y | X^{S^\prime}$.

Firstly, elements in $dis(Y)$ should be in $S^\prime$. If $i_n \in dis(Y)$, there is a path between $X^{i_n}$ and $Y$ as 
$$
X^{i_n} \leftrightarrow X^{i_{n-1}} \leftrightarrow \cdots \leftrightarrow X^{i_1} \leftrightarrow Y.
$$
It has $i_1 \in S^\prime$ since $i_1 \in sp(Y)$. Then the path $X^{i_2} \leftrightarrow X^{i_1} \leftrightarrow Y$ can not be $m$-blocked given $X^{S^\prime}$ since $i_1 \in S^\prime$, which means $i_2 \in S^\prime$. Recursively, we can get $i_n \in S^\prime$. 

Secondly, if $k \in pa(dis(Y))$, there is a path between $X^k$ and $Y$ as 
$$
X^k \rightarrow X^{i_n} \leftrightarrow X^{i_{n-1}} \leftrightarrow \cdots \leftrightarrow X^{i_1} \leftrightarrow Y.
$$
Because $i_1,\cdots,i_n \in S^\prime$, this path can not be $m$-blocked given $X^{S^\prime}$. So $k \in S^\prime$, this means that $pa(dis(Y)) \subseteq S\prime$. 

Thirdly, we consider the nodes in the district of $Y$'s children. If $i_n \in dis(ch(Y))$, there is a path connecting $X^{i_n}$ and $Y$ as 
$$
X^{i_n} \leftrightarrow X^{i_{n-1}} \leftrightarrow \cdots \leftrightarrow X^{i_1} \leftarrow Y.
$$
As $i_1 \in S^\prime$, the path $X^{i_2} \leftrightarrow X^{i_1} \leftarrow Y$ can not be $m$-blocked given $X^{S^\prime}$, implying $i_2 \in S^\prime$. Hence $i_n \in S^\prime$ by induction.

Lastly, given a $k \in pa(dis(ch(Y)))$, there is a path as 
$$
X^k \rightarrow X^{i_n} \leftrightarrow X^{i_{n-1}} \leftrightarrow \cdots \leftrightarrow X^{i_1} \leftarrow Y.
$$ 
Since $i_1,\cdots,i_n \in S^\prime$, the path can not be $m$-blocked given $X^{S^\prime}$. We can obtain $k \in S^\prime$, which means that $pa(dis(ch(Y))) \subseteq S^\prime$.

Therefore, $S \subseteq S^\prime$. Since we have already proved that $X^S$ can $m$-separate $Y$ and other predictors, $MB(Y) = S$. So we characterize the Markov blanket of $Y$ by four kinds of nodes in the graph.
\end{proof}

Since the Markov blanket of $Y$ can $m$-separate $Y$ and the remaining predictors, we can get conditional independence relations $Y \indep X^{\{1,\cdots,p\}\backslash MB(Y)}| X^{MB(Y)}$ by Markov property in the ADMG. Thus the Markov blanket is the most predictive set of predictors because it is not necessary to add other predictors. 

However, when multiple different environments exist, in other words, there are interventions, the Markov blanket may not be perfect since the dependence of $Y$ on predictors in the Markov blanket of $Y$ may change across interventional environments. We need predictors which are both intervention-stable and the most predictive. More specifically, the set of predictors which contains the most information about $Y$ and can explain the variability in the interventions is what we aim to determine from the ADMG.

\subsection{Intervention Stability in ADMGs}

\subsubsection{Intervention-Stable Set}

\begin{definition}[Intervention-Stable]
In the graphical models of Setting \ref{Setting 2}, a set $S \subseteq \{1,\cdots,p\}$ is called intervention-stable if for all $l \in \{1,\cdots,m\}$, it holds that $I^l \indep_m Y | X^S$.
\end{definition}

\begin{definition}[Intervened Sub-district]
\label{subdistrict}
Assume there is an intervention $I$ acting on the district of a child $X^{i_1}$ of Y. Then an intervened sub-district is a series of nodes $(i_1, \cdots, i_n)$ which can be connected as 
\begin{equation}
    I \rightarrow X^{i_n} \twoset{-}{\leftrightarrow} X^{i_{n-1}} \twoset{-}{\leftrightarrow} \cdots \twoset{-}{\leftrightarrow} X^{i_1} \leftarrow Y \label{intsub}
\end{equation}
where $i_1,\cdots,i_n \in \{1,\cdots,p\}$.
\end{definition}

\begin{definition}[Sub-district Collider]
\label{SubCol}
A node $X^{i_k}$ on an intervened sub-district \eqref{intsub} is a sub-district collider on the intervened sub-district if it satisfies
\begin{enumerate}[(i)]
    \item there are no directed arrows from $X^{i_k}$ to $X^{i_{k-1}}$ and $X^{i_{k+1}}$, i.e., $(X^{i_{k+1}}, X^{i_k},X^{i_{k-1}})$ is one of the cases in Table 1, and
    \item there is no directed path from the node $X^{i_k}$ to the response $Y$.
\end{enumerate}
\end{definition}

We define $X^{i_0} = Y$ and $X^{i_{n+1}} = I$ by default. And we also say $i_k$ is a sub-district collider on the intervened sub-district if $X^{i_k}$ is a sub-district collider on that.

Using definitions \ref{subdistrict} and \ref{SubCol}, we can show that there must be a sub-district collider on each intervened sub-district.

\begin{proposition}
\label{sub-district collider}
Given any intervened sub-district \eqref{intsub}, then there is at least one sub-district collider on this intervened sub-district.
\end{proposition}

\begin{proof}
Firstly, as the arrows of two ends are $I \rightarrow$ and $\leftarrow Y$, there is an $i_k$, $k \in \{1,\cdots,n\}$ which satisfies the first condition of the sub-district collider. If there is no directed path from $X^{i_k}$ to Y, then $X^{i_k}$ is a sub-district collider. 

Otherwise, assume there is a directed path from $X^{i_k}$ to $Y$. Since the ADMG has no cycles, there can not be a directed path from $Y$ to $X^{i_k}$. So there must exist one node $X^{i_j}$ between $X^{i_k}$ and $Y$, which absorbs all arrowheads from adjacent nodes on the intervened sub-district. If there is no directed path from $X^{i_j}$ to $Y$, then $X^{i_j}$ is a sub-district collider. Otherwise, we can replicate the procedure of how to find $X^{i_j}$ until one sub-district collider appears on the intervened sub-district. 

Therefore, the existence of a sub-district collider on every intervened sub-district is guaranteed.
\end{proof}

The sub-district collider is crucial for separation from interventions and the response. Indeed, we can show that any set that can $m$-block all paths between $I$ and $Y$ in an intervened sub-district does not contain a sub-district collider nor its descendants.

\begin{proposition}
\label{remove sub hidden}
For each intervened sub-district 
$$
I \rightarrow X^{i_n} \twoset{-}{\leftrightarrow} X^{i_{n-1}} \twoset{-}{\leftrightarrow} \cdots \twoset{-}{\leftrightarrow} X^{i_1} \leftarrow Y,
$$
if a condition set $S$ can $m$-block all paths between $I$ and $Y$, then there exists a sub-district collider such that itself and its descendants are not in the condition set.
\end{proposition}

\begin{proof}
Firstly, to $m$-block 
$$
I \rightarrow X^{i_n} \leftrightarrow X^{i_{n-1}} \leftrightarrow \cdots \leftrightarrow X^{i_1} \leftarrow Y,
$$
there must be a collider such that itself and its descendants are not in the condition set $S$. Let $X^{i_k}$ be the farthest node away from $Y$ such that $i_k \notin an(S) \cup S$. If there is no directed path from $X^{i_k}$ to $Y$, and $X^{i_k}$ absorbs all arrowheads from adjacent nodes in the intervened sub-district, then $X^{i_k}$ is a sub-district collider on the intervened sub-district. 

Otherwise, assume there is a directed path from $X^{i_k}$ to $Y$ as $X^{i_k} \rightarrow \cdots \rightarrow Y$. Since $X^{i_k}$ is the farthest node away from $Y$ such that itself and its descendants are not in the condition set $S$, it implies that either $i_h \in S$ or $de(X^{i_h}) \cap S \neq \emptyset$, $h = k+1, \cdots, n$. Besides, nodes on $X^{i_k} \rightarrow \cdots \rightarrow Y$ are not in $S$. Thus, the path $I \rightarrow X^{i_n} \leftrightarrow \cdots \leftrightarrow X^{i_k} \rightarrow \cdots \rightarrow Y$ can not be $m$-blocked by $S$, contradicting the condition. So there can not be a directed path from $X^{i_k}$ to $Y$.

Therefore, we only consider whether an arrow goes out of $X^{i_k}$ on the intervened sub-district. There are two situations.

\begin{enumerate}[(i)]
    \item $I \rightarrow X^{i_n} \twoset{-}{\leftrightarrow} \cdots \twoset{\leftarrow}{\leftrightarrow} X^{i_k} \twoset{-}{\leftrightarrow} \cdots \twoset{-}{\leftrightarrow} X^{i_1} \leftarrow Y$, 
    
    which means there is an arrow starting from $X^{i_k}$ pointing to $X^{i_{k+1}}$. The leftward arrows starting from $X^{i_k}$ will end at $X^{i_s}$ where $s > k$. Since $X^{i_s}$ is a descendant of $X^{i_k}$, $X^{i_s}$ and its descendants are not in $S$. This contradicts the farthest property of $X^{i_k}$.
    
    \item $I \rightarrow X^{i_n} \twoset{-}{\leftrightarrow} \cdots \twoset{-}{\leftrightarrow} X^{i_k} \twoset{\rightarrow}{\leftrightarrow} \cdots \twoset{-}{\leftrightarrow} X^{i_1} \leftarrow Y$,
    
    which means there is an arrow starting from $X^{i_k}$ pointing to $X^{i_{k-1}}$. The rightward arrows starting from $X^{i_k}$ will end at $X^{i_s}$ where $s < k$. Since $X^{i_s}$ is a descendant of $X^{i_k}$ and has no arrow out of it, $X^{i_s}$ satisfies the first condition of sub-district collider Definition \ref{SubCol}. Then we consider the second condition in Definition \ref{SubCol}. If there is a directed path from $X^{i_s}$ to $Y$ as $X^{i_s} \rightarrow \cdots \rightarrow Y$, then the path 
    $$
    I \rightarrow X^{i_n} \leftrightarrow \cdots \leftrightarrow X^{i_k} \rightarrow \cdots \rightarrow X^{i_s} \rightarrow \cdots \rightarrow Y
    $$
    can not be $m$-blocked by $S$, because $X^{i_n},\cdots,X^{i_{k+1}}$ are not qualified colliders and nodes on 
    $$
    X^{i_k} \rightarrow \cdots \rightarrow X^{i_s} \rightarrow \cdots \rightarrow Y
    $$
    are not in the conditional set $S$. This contradicts the condition that $S$ can $m$-separate all paths between $I$ and $Y$. So there is no directed path from $X^{i_s}$ to $Y$, which means $X^{i_s}$ satisfies the second condition in Definition \ref{SubCol}. Thus it is a sub-district collider on the intervened sub-district.
\end{enumerate}

From the above discussion, we can see that either $X^{i_k}$ or $X^{i_s}$ is a sub-district collider. Therefore, at least one sub-district collider exists such that itself and its descendants are not in the conditional set $S$.
\end{proof}

We now give a construction of an intervention-stable set.

\begin{definition}[Complete Set of Sub-district Colliders]
Given the graphical models of Setting \ref{Setting 2} and $s^1, s^2, \cdots, s^l \subseteq \{1,\cdots,p\}$ all intervened sub-districts. We call a subset of nodes $C \subseteq \{1,\cdots,p\}$ a \textit{complete set of sub-district colliders} if 
$$
C = \{c^{1,1},\cdots,c^{1,i_1},c^{2,1},\cdots,c^{2,i_2},\cdots,c^{l,1},\cdots,c^{l,i_l}\},
$$
where $c^{1,1},\cdots,c^{1,i_1},c^{2,1},\cdots,c^{2,i_2},\cdots,c^{l,1},\cdots,c^{l,i_l}$ are sub-district colliders satisfying that $c^{h,\cdot} \in s^h$ and $i_h \geq 1,\ h = 1,\cdots,l$.
\end{definition}


\begin{definition}[Intervention Set]
Given the graphical models of Setting~\ref{Setting 2}, let 
$s^1,\ldots,s^l\subseteq \{1,\ldots,p\}$ be all intervened sub-districts. 
For each complete set of sub-district colliders
\[
C=
\{c^{h,k}: h=1,\ldots,l,\ k=1,\ldots,i_h\},
\]
we define the intervention set $N^{int}(C)$ by
\[
N^{int}(C)
:=
\bigcup_{h=1}^l
\bigcup_{k=1}^{i_h}
\left(
de(X^{c^{h,k}})\cup \{c^{h,k}\}
\right).
\]
\end{definition}

\begin{theorem}
\label{theorem 1}
Given the graphical models in Setting \ref{Setting 2}, let $C$ be a complete set of sub-district colliders, $N^{int}(C)$ be the intervention set determined by $C$ and
$$
N^{-int}(C) = \{1,\cdots,p\} \backslash N^{int}(C).
$$
Then $N^{-int}(C)$ is an intervention-stable set with respect to all interventions. 
\end{theorem}

\begin{proof}
We need to show that $\forall l \in \{1,\cdots,m\}$, it holds that $I^l \indep_m Y | X^{N^{-int}(C)}$. So we look at all paths between interventions and the response $Y$. It is necessary to divide these paths into six shapes in terms of how the path enters $Y$. Fix $l \in \{1,\cdots,m\}$ and let $P$ be a path connecting $I^l$ and $Y$. 

\begin{enumerate}[(i)]
    \item $I^l \rightarrow \cdots\ X^k \rightarrow Y$, 
    
    where $k \notin sp(Y)$. $P$ enters $Y$ via a parent but not a spouse of $Y$. It has $k \in N^{-int}(C)$ otherwise, there is a directed path from a sub-district collider to $Y$, contradicting the definition of a sub-district collider. Thus, $P$ can be $m$-blocked by $X^{N^{-int}(C)}$ as $X^k$ is a non-collider on $P$.
    
    \item $I^l \rightarrow \cdots\ X^k \rightarrow X^{i_n} \twoset{-}{\leftrightarrow} X^{i_{n-1}} \twoset{-}{\leftrightarrow} \cdots \twoset{-}{\leftrightarrow} X^{i_1} \twoset{-}{\leftrightarrow} Y$,
    
    where $k \notin sp(X^{i_n})$. The path first enters a parent of $dis(Y)$. If $k \in N^{-int}(C)$, $P$ is $m$-blocked by $X^{N^{-int}(C)}$ as $X^k$ is a non-collider on $P$. If not, $X^k$ is either a descendant of a sub-district collider or a sub-district collider. Then there must be a collider on the path between $X^k$ and $Y$. Otherwise, there will be a direct path from a sub-district collider to $Y$. Let $X^{i_q}$ be the closest collider to $X^k$, which means that $i_q \in N^{int}(C)$ and then $de(X^{i_q}) \subseteq N^{int}(C)$. So $P$ can be $m$-blocked by $X^{N^{-int}(C)}$ as $X^{i_q}$ is a collider on $P$.
    
    \item $I^l \rightarrow \cdots\ X^k \leftarrow X^{i_n} \twoset{-}{\leftrightarrow} X^{i_{n-1}} \twoset{-}{\leftrightarrow} \cdots \twoset{-}{\leftrightarrow} X^{i_1} \twoset{-}{\leftrightarrow} Y$,
    
    where $k \notin sp(X^{i_n})$. $P$ firstly enters a child of $dis(Y)$. If $i_n \in N^{-int}(C)$, $P$ is $m$-blocked by $X^{N^{-int}(C)}$ as $X^{i_n}$ is a non-collider on $P$. Otherwise, there must be a collider on the path between $X^{i_n}$ and $I^l$ as two end arrows have different directions. Similarly, let $X^q$ be the closest collider to $X^{i_n}$, so $q \in N^{int}(C)$ and $de(X^q) \subseteq N^{int}(C)$. Thus $P$ is still $m$-blocked by $X^{N^{-int}(C)}$ as $X^q$ is a collider on $P$.
    
    \item $I^l \rightarrow X^{i_n} \twoset{-}{\leftrightarrow} X^{i_{n-1}} \twoset{-}{\leftrightarrow} \cdots \twoset{-}{\leftrightarrow} X^{i_1} \leftarrow Y$ 
    
    where $i_1 \notin sp(Y)$. In this case, the intervention acts on the district of a child of $Y$ directly. From the construction of $N^{int}(C)$, there is a sub-district collider $X^{i_q}$ on the intervened sub-district such that $i_q \in N^{int}(C)$ and $de(X^q) \subseteq N^{int}(C)$. By Proposition \ref{Block}, $P$ is $m$-blocked by $X^{N^{-int}(C)}$.
    
    \item $I^l \rightarrow \cdots X^k \rightarrow X^{i_n} \twoset{-}{\leftrightarrow} X^{i_{n-1}} \twoset{-}{\leftrightarrow} \cdots \twoset{-}{\leftrightarrow} X^{i_1} \leftarrow Y$, 
    
    where $i_1 \notin sp(Y)$ and $k \notin sp(X^{i_n})$. $P$ enters the district of a child of $Y$ via a parent of this district. Then if $k \in N^{-int}(C)$, $P$ is $m$-blocked by $X^{N^{-int}(C)}$ as $X^k$ is a non-collider on $P$. Otherwise, $k \in N^{int}(C)$ and $de(X^k) \subseteq N^{int}(C)$. Moreover, there must be a collider on the path between $X^k$ and $Y$ as the two ends are $X^k \rightarrow$ and $\leftarrow Y$. Let $X^{i_q}$ be the closest collider to $X^k$, then $i_q \in N^{int}(C)$ and $de(X^{i_q}) \subseteq N^{int}(C)$. Thus $P$ is $m$-blocked by $X^{N^{-int}(C)}$ as $X^{i_q}$ is a collider on $P$.
    
    \item $I^l \rightarrow \cdots X^k \leftarrow X^{i_n} \twoset{-}{\leftrightarrow} X^{i_{n-1}} \twoset{-}{\leftrightarrow} \cdots \twoset{-}{\leftrightarrow} X^{i_1} \leftarrow Y$, 
    
    where $i_1 \notin sp(Y)$ and $k \notin sp(X^{i_n})$. Then if $i_n \in N^{-int}(C)$, $P$ is $m$-blocked by $X^{N^{-int}(C)}$ as $X^{i_n}$ is a non-collider on $P$. Otherwise, there must be a collider on the path between $I^l$ and $X^{i_n}$. Let $X^q$ be the closest collider to $X^{i_n}$. So $X^q$ is a descendant of $X^{i_n}$, thus a descendant of a sub-district collider. $q \notin N^{-int}(C)$ and $de(X^q) \cap N^{-int}(C) = \emptyset$, so $P$ is still $m$-blocked by $X^{N^{-int}(C)}$ as $X^q$ is a collider on $P$.
\end{enumerate}

Therefore, all paths from $I^l$ to $Y$ are $m$-blocked by $X^{N^{-int}(C)}$, which means $I^l \indep_m Y | X^{N^{-int}(C)}$, $\forall l \in \{1,\cdots,m\}$. That is to say, $N^{-int}(C)$ is an intervention-stable set with respect to all interventions. 
\end{proof}

\subsubsection{Stable Blanket in ADMGs}

\begin{definition}[Stable Frontier]
For each complete set of sub-district colliders $C$, we can define a stable frontier, denoted by $SF_I(Y, C)$, as the smallest node set $S \subseteq N^{-int}(C) = \{1,\cdots,p\} \backslash N^{int}(C)$ that satisfies
$$
\forall j \in N^{-int}(C) \backslash S :\ X^j \indep_m Y | X^S.
$$
\end{definition}

We establish an example in Figure \ref{fig:M3} to illustrate that the stable frontiers are likely to be different depending on the complete set of sub-district colliders.

\begin{example}
Assume SCMs over $(X^1,\cdots,X^8,Y,I^1,I^2)$.
\begin{figure}[H]
    \centering
    \begin{tikzpicture}[
        >=Stealth,
        thick,
        scale=1.15,
        every node/.style={font=\normalsize},
        obsnode/.style={
            circle,
            draw=black,
            fill=white,
            minimum size=9mm,
            inner sep=0pt
        },
        intnode/.style={
            rectangle,
            draw=black,
            fill=white,
            minimum width=9mm,
            minimum height=9mm,
            inner sep=0pt
        },
        ynode/.style={
            double,
            circle,
            draw=black,
            fill=white,
            minimum size=10mm,
            inner sep=0pt,
            line width=0.8pt
        }
    ]

    \node[intnode] (i1) at (0,0) {$I^1$};
    \node[obsnode] (x2) at (2,0) {$X^2$};
    \node[obsnode] (x3) at (4,0) {$X^3$};
    \node[obsnode] (x4) at (6,0) {$X^4$};
    \node[ynode]   (y)  at (8,0) {$Y$};

    \node[obsnode] (x8) at (4,1.8) {$X^8$};
    \node[obsnode] (x1) at (10,1.8) {$X^1$};

    \node[intnode] (i2) at (4,-2) {$I^2$};
    \node[obsnode] (x7) at (6,-2) {$X^7$};
    \node[obsnode] (x6) at (8,-2) {$X^6$};
    \node[obsnode] (x5) at (10,-2) {$X^5$};

    \draw[->] (i1) -- (x2);
    \draw[->] (x8) -- (x3);
    \draw[->] (y) -- (x4);
    \draw[->] (x1) -- (y);
    \draw[->] (x2) -- (x7);
    \draw[->] (i2) -- (x7);
    \draw[->] (x7) -- (x6);

    \draw[<->, bend left=18] (x2) to (x3);
    \draw[->,  bend left=18] (x3) to (x2);

    \draw[<->, bend left=18] (x3) to (x4);
    \draw[->,  bend right=18] (x3) to (x4);

    \draw[<->] (y) -- (x6);
    \draw[<->] (y) -- (x5);

    \end{tikzpicture}
    \caption{ADMG generated by latent projection with respect to hidden variables. $I^1$ and $I^2$ are interventions. There are two sub-district colliders on the intervened sub-district $I^1 \rightarrow X^2 \twoset{\leftarrow}{\leftrightarrow} X^3 \twoset{\rightarrow}{\leftrightarrow} X^4 \leftarrow Y$. If the complete set of colliders chooses only $X^2$, then $SF_I(Y,C) = \{X^1,X^3,X^4,X^5,X^8\}$. If the complete set of colliders chooses only $X^4$, then $SF_I(Y,C) = \{X^1,X^5,X^6,X^7\}$. Thus the stable frontiers are not unique.}
    \label{fig:M3}
\end{figure}
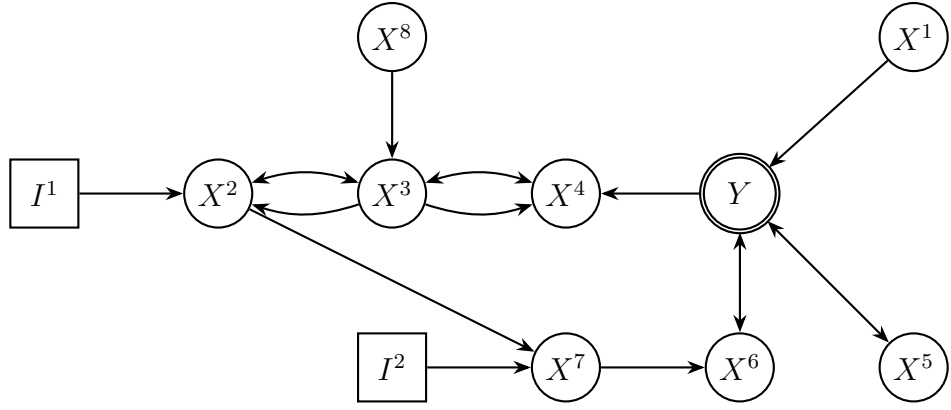
\end{example}

The following theorem shows that for each complete set of sub-district colliders $C$, the corresponding stable frontier is well defined.

\begin{theorem}
Given the graphical models of Setting \ref{Setting 2}, the generated ADMG is $\mathcal{G}$. The stable frontier of $Y$ given a complete set of sub-district colliders $C$ consists of the district of Y, the parents of the district of Y, the districts of Y's children, and the parents of districts of Y's children in the sub-graph of $\mathcal{G}$ over $(N^{-int}(C), Y, I)$. Moreover, it is intervention-stable with respect to all interventions.
\end{theorem}

\begin{proof}
We let $S = dis(Y) \cup pa(dis(Y)) \cup dis(ch(Y)) \cup pa(dis(ch(Y))) \backslash Y$ in the sub-graph $\mathcal{G}_{(N^{-int}(C), Y, I)}$. Firstly, we prove that for $j \in N^{-int}(C) \backslash S$, $X^j \indep_m Y | X^S$. We look at the paths connecting $X^j$ and $Y$. Let $P$ be a path between $X^j$ and $Y$:

\begin{enumerate}[(i)]
    \item $X^j \cdots X^k \rightarrow Y$,
    
    with $k \notin sp(Y)$. $P$ is $m$-blocked by $S$ as $k \in pa(Y) \subseteq S$.
    
    \item $X^j \twoset{-}{\leftrightarrow} X^{i_n} \twoset{-}{\leftrightarrow} \cdots \twoset{-}{\leftrightarrow} X^{i_1} \twoset{-}{\leftrightarrow} Y$.
    
    In this case, as $j \in N^{-int}(C)$ and $j \notin S$, there must be an $i_q$ such that $i_q \notin N^{-int}(C)$ and $1 \leq q \leq n$. Let $q$ be the smallest integer in $\{1, \cdots, n\}$ satisfying $i_q \in N^{int}(C)$. Then there is no directed arrow from $X^{i_q}$ to $X^{i_{q-1}}$. Moreover, let $i_l$ be the farthest descendant of $X^{i_q}$ such that $q \leq l \leq n$. Thus $X^{i_l}$ is a collider on the path $P$. Besides, it holds that $i_l \in N^{int}(C)$ and $de(X^{i_l}) \subseteq N^{int}(C)$, which means that $i_l \notin S$ and $de(X^{i_l}) \cap S = \emptyset$. So $P$ is $m$-blocked by $S$ as $X^{i_l}$ acts as a collider on $P$.
    
    \item $X^j \cdots X^k \rightarrow X^{i_n} \twoset{-}{\leftrightarrow} \cdots X^{i_1} \twoset{-}{\leftrightarrow} Y$,
    
    where $k \notin sp(X^{i_n})$. If $j = k$, the argument is similar to that in (ii) since $j \notin S$. There must be an $i_q$ such that $i_q \in N^{int}(C)$ and $1 \leq q \leq n$. Thus $P$ is $m$-blocked by $S$ in this case.
    
    If $j \neq k$, we need to discuss three situations. When $k \in S$, $P$ is $m$-blocked by $S$ as $X^k$ is a non-collider on $P$. Otherwise, if $k \in N^{int}(C)$, let $X^{i_q}$ be the farthest descendant of $X^k$ such that $1 \leq q \leq n$. $X^{i_q}$ is a collider on $P$ as $Y$ can not be a descendant of $X^k$. So $P$ can be $m$-blocked by $S$. If $k \in N^{-int}(C)$ but $k \notin S$, there must be an $i_q$ such that $i_q \in N^{int}(C)$ and $1 \leq q \leq n$. Analogously to the argument in (ii), $P$ is still $m$-blocked by $S$.
    
    \item $X^j \cdots X^k \leftarrow X^{i_n} \twoset{-}{\leftrightarrow} \cdots X^{i_1} \twoset{-}{\leftrightarrow} Y$,
    
    where $k \notin sp(X^{i_n})$. If $i_n \in S$, then $P$ is $m$-blocked by $S$ since $X^{i_n}$ is a non-collider on $P$.
    
    If $i_n \notin S$, moreover, if $i_n \in N^{int}(C)$, there must be a collider between $X^{i_n}$ and $X^j$ as $j \notin N^{int}(C)$. Let $X^l$ be farthest descendant of $X^{i_n}$ located between $X^k$ and $X^j$. So $P$ is $m$-blocked by $S$ as $X^l$ is a collider on $P$. Besides, if $i_n \in N^{-int}(C)$ but $i_n \notin S$, there will be an $i_q$ such that $i_q \in N^{int}(C)$ and $1 \leq q \leq n$. Likewise, we can get $P$ is $m$-blocked by $S$.
    
    \item $X^j \twoset{-}{\leftrightarrow} X^{i_n} \twoset{-}{\leftrightarrow} \cdots \twoset{-}{\leftrightarrow} X^{i_1} \leftarrow Y$.
    
    As $j \notin S$, there must be an $i_q$ satisfying $i^q \in N^{int}(C)$ and $1 \leq q \leq n$. The argument is similar to that in (ii). So $P$ is $m$-blocked by $S$.
    
    \item $X^j \cdots X^k \rightarrow X^{i_n} \twoset{-}{\leftrightarrow} \cdots \twoset{-}{\leftrightarrow} X^{i_1} \leftarrow Y$,
    
    where $k \notin sp(X^{i_n})$. If $j = k$, there must be an $i_q$ such that $i^q \in N^{int}(C)$ and $1 \leq q \leq n$. Likewise, $P$ is $m$-blocked by $S$.
    
    If $j \neq k$, in addition, $k \in S$, then $P$ is $m$-blocked by $S$ as $X^k$ is a non-collider on $P$. If $k \notin S$, there are two cases. One is that $k \in N^{int}(C)$, under which let the farthest descendant of $X^k$ be $X^{i_q}$, $1 \leq q \leq n$. So $i_q \in N^{int}(C)$ and $de(X^{i_q}) \subseteq N^{int}(C)$. We can get that $P$ is $m$-blocked by $S$ because $X^{i_q}$ is a collider on $P$. The other is that $k \in N^{-int}(C) \backslash S$. Then there must be an $i_q$ such that $i^q \in N^{int}(C)$ and $1 \leq q \leq n$. Likewise, $P$ is $m$-blocked by $S$.
    
    \item $X^j \cdots X^k \leftarrow X^{i_n} \twoset{-}{\leftrightarrow} \cdots \twoset{-}{\leftrightarrow} X^{i_1} \leftarrow Y$,
    
    where $k \notin sp(X^{i_n})$. When $i_n \in S$, $P$ is $m$-blocked by $S$ as $X^{i_n}$ is a non-collider.
    
    If $k \in N^{int}(C)$, there must be a collider between $X^j$ and $X^k$ because $j \in N^{-int}(C)$. So $P$ is $m$-blocked by $S$. Otherwise, $k \in N^{-int}(C) \backslash S$, which will result in an $i_q \in N^{int}(C)$, $1 \leq q \leq n-1$. Similarly, $P$ is $m$-blocked by $S$.
\end{enumerate}

Secondly, we aim to show that for each $S^\prime \subseteq N^{-int}(C)$ satisfying $\forall j \in N^{-int}(C) \backslash S^\prime$, $X^j \indep_m Y | X^{S^\prime}$, it holds that $S \subseteq S^\prime$. It suffices to prove that all the four parts of $S$ are contained by $S^\prime$. These arguments are analogous to the second part proof of Proposition \ref{MB hidden}. Therefore, $SF_I(Y, C)$ is the union of the district of $Y$, the parents of the district of $Y$, the districts of $Y$'s children, and the parents of the districts of $Y$'s children in the sub-graph $\mathcal{G}_{(N^{-int}(C), Y, I)}$.
 
Thirdly, parallel to the proof of Theorem \ref{theorem 1}, let $P$ be a path from $I^l$ to $Y$, $l \in \{1,\cdots,m\}$. Then $P$ has one of the following shapes.

\begin{enumerate}[(i)]
    \item $I^l \rightarrow \cdots\ X^k \rightarrow Y$, 
    
    where $k \notin sp(Y)$. Then $P$ is $m$-blocked by $X^{SF_I(Y,C)}$ since $k \in pa(Y) \subseteq SF_I(Y,C)$.
    
    \item  $I^l \rightarrow \cdots\ X^k \rightarrow X^{i_n} \twoset{-}{\leftrightarrow} X^{i_{n-1}} \twoset{-}{\leftrightarrow} \cdots \twoset{-}{\leftrightarrow} X^{i_1} \twoset{-}{\leftrightarrow} Y$,
    
    where $k \notin sp(X^{i_n})$. Let $X^{i_j}$ be the farthest node away from $Y$ on the district such that $i_1, \cdots, i_j \in N^{-int}(C)$ but $i_{j+1} \notin N^{-int}(C)$ (if $j < n$). If $j = n$, then $k \in N^{-int}(C)$, which means $k \in pa(dis(Y))$ in the sub-graph $\mathcal{G}_{N^{-int}(C)}$. So $P$ is $m$-blocked by $X^{SF_I(Y,C)}$ as $X^k$ is a non-collider on $P$. 
    
    Otherwise, if $j < n$, we look at $X^{i_{j+1}}$. $i_{j+1} \in N^{int}(C)$ so $de(X^{i_{j+1}}) \subseteq N^{int}(C)$. There must be a collider between $X^k$ and $X^{i_j}$ because the middle part of $P$ is $X^k \rightarrow \cdots \leftrightarrow X^{i_j}$ or $X^k \rightarrow \cdots \leftarrow X^{i_j}$. Let $X^{i_q}$ be the closest collider to $X^{i_j}$, so $i_q \in de(X^{i_{j+1}})$ or $q = j+1$. Thus, $i_q \notin N^{-int}(C) \cup an(N^{-int}(C))$, then $i_q \notin SF_I(Y,C) \cup an(SF_I(Y,C))$. So $P$ is $m$-blocked given $X^{SF_I(Y,C)}$ as $X^{i_q}$ is a collider on $P$.

    \item $I^l \rightarrow \cdots\ X^k \leftarrow X^{i_n} \twoset{-}{\leftrightarrow} X^{i_{n-1}} \twoset{-}{\leftrightarrow} \cdots \twoset{-}{\leftrightarrow} X^{i_1} \twoset{-}{\leftrightarrow} Y$,
    
    where $k \notin sp(X^{i_n})$. Let $X^{i_j}$ be the farthest node away from $Y$ on the district such that $i_1, \cdots, i_j \in N^{-int}(C)$ but $i_{j+1} \notin N^{-int}(C)$ (if $j < n$). If $j = n$, then $P$ is $m$-blocked by $X^{SF_I(Y,C)}$ since $X^{i_n}$ is a non-collider on $P$ such that $i_n \in dis(Y) \subseteq SF_I(Y,C)$ in the sub-graph over $N^{-int}(C)$.
    
    Otherwise, $j < n$, then $i_{j+1} \in N^{int}(C)$. And there must be a collider between $X^{i_j}$ and $I^l$ since $P$ has a part as $I^l \rightarrow \cdots \leftrightarrow X^{i_j}$ or $I^l \rightarrow \cdots \leftarrow X^{i_j}$. Let $X^q$ be the closest collider to $X^{i_j}$, hence $q \in de(X^{i_{j+1}})$ or $q = i_{j+1}$, which means $q \notin N^{-int}(C) \cup an(N^{-int}(C))$. So $P$ is also $m$-blocked by $X^{SF_I(Y,C)}$ as $X^q$ is a collider on $P$.
    
    \item $I^l \rightarrow X^{i_n} \twoset{-}{\leftrightarrow} X^{i_{n-1}} \twoset{-}{\leftrightarrow} \cdots \twoset{-}{\leftrightarrow} X^{i_1} \leftarrow Y$, 
    
    where $i_1 \notin sp(Y)$. From the construction of $N^{int}(C)$, there is a sub-district collider $X^{i_j}$ such that $i_j \in N^{int}(C)$ and $de(X^{i_j}) \subseteq N^{int}(C)$. Then $P$ is $m$-blocked by $X^{N^{-int}(C)}$ as $X^{i_j}$ is a collider on $P$, so $P$ is $m$-blocked by $X^{SF_I(Y,C)}$.
    
    \item $I^l \rightarrow \cdots X^k \rightarrow X^{i_n} \twoset{-}{\leftrightarrow} X^{i_{n-1}} \twoset{-}{\leftrightarrow} \cdots \twoset{-}{\leftrightarrow} X^{i_1} \leftarrow Y$, 
    
    where $i_1 \notin sp(Y)$ and $k \notin sp(X^{i_n})$. Let $X^{i_j}$ be the farthest node away from $X^{i_1}$ on the district of $X^{i_1}$ such that $i_1, \cdots, i_j \in N^{-int}(C)$ but $i_{j+1} \notin N^{-int}(C)$ (if $j < n$). If $j = n$, then $k \in N^{-int}(C)$. Furthermore, $k \in pa(dis(ch(Y)))$ in the sub-graph $\mathcal{G}_{N^{-int}(C)}$, which indicates that $P$ is $m$-blocked by $X^{SF_I(Y,C)}$ as $X^k$ is a non-collider on $P$. 
    
    Otherwise, if $j < n$, then $i_{j+1} \in N^{int}(C)$. And there must be a collider between $X^{i_j}$ and $X^k$ since $P$ has a part as $X^k \rightarrow \cdots \leftarrow X^{i_j}$ or $X^k \rightarrow \cdots \leftrightarrow X^{i_j}$. Let $X^{i_q}$ be the closest collider to $X^{i_j}$. It indicates that $i_q \in de(X^{i_{j+1}})$ or $q = j+1$, which means $i_q \notin N^{-int}(C) \cup an(N^{-int}(C))$. This implies that $P$ is $m$-blocked given $X^{SF_I(Y,C)}$ as $X^{i_q}$ is a collider on $P$.
    
    \item $I^l \rightarrow \cdots X^k \leftarrow X^{i_n} \twoset{-}{\leftrightarrow} X^{i_{n-1}} \twoset{-}{\leftrightarrow} \cdots \twoset{-}{\leftrightarrow} X^{i_1} \leftarrow Y$, 
    
    where $i_1 \notin sp(Y)$ and $k \notin sp(X^{i_n})$. Still let $X^{i_j}$ be the farthest node away from $X^{i_1}$ on the district of $X^{i_1}$ such that $i_1, \cdots, i_j \in N^{-int}(C)$ but $i_{j+1} \notin N^{-int}(C)$ (if $j < n$). If $j = n$, then $P$ is $m$-blocked by $X^{SF_I(Y,C)}$ as $X^{i_n}$ is a non-collider on $P$ and $i_n \in dis(ch(Y)) \subseteq SF_I(Y,C)$ in the sub-graph $\mathcal{G}_{N^{-int}(C)}$. 
    
    Otherwise, $i_{j+1} \in N^{int}(C)$ when $j < n$. And there must be a collider between $X^{i_j}$ and $I^l$ since $P$ has a part as $I^l \rightarrow \cdots \leftarrow X^{i_j}$ or $I^l \rightarrow \cdots \leftrightarrow X^{i_j}$. Let $X^{q}$ be the closest collider to $X^{i_j}$. It indicates that $q \in de(X^{i_{j+1}})$ or $q = i_{j+1}$, which means $i_q \notin N^{-int}(C) \cup an(N^{-int}(C))$. We hence obtain that $P$ is $m$-blocked given $X^{SF_I(Y,C)}$ as $X^q$ is a collider on $P$.
\end{enumerate}

From the above discussions in terms of the way how the path enters $Y$, we know that the stable frontier $SF_I(Y,C)$ for any complete set of sub-district colliders $C$ is also intervention-stable with respect to all interventions. 
\end{proof}

Since any intervention-stable set can $m$-separate all interventions and the response, for each intervened sub-district, there is at least one sub-district collider such that itself and its descendants are not in the intervention-stable set. By Theorem \ref{theorem 1}, by removing every complete set of sub-district colliders $C$ and their descendants, we can construct an intervention-stable set $N^{-int}(C)$. Moreover, the stable frontier $SF_I(Y, C)$ is the most informative subset of the intervention-stable set $N^{-int}(C)$ since it $m$-separates all other variables and is still intervention-stable. 

Furthermore, since each complete set of sub-district colliders $C$ can lead to an intervention-stable set $N^{-int}(C)$, $N^{-int}(C)$ can differ considerably in size and elements as $C$ varies. Usually, the larger the cardinality of $N^{-int}(C)$ is, the more information it may contain. Thus it is better to choose one sub-district collider for every intervened sub-district. Moreover, by choosing the farthest sub-district colliders away from $Y$, the districts of $Y$'s children seem to be longer in the sub-graph over $N^{-int}(C)$. Then it is possible to contain more variables in the corresponding stable frontier.

\begin{corollary}
Given the graphical models of Setting \ref{Setting 2}, let $C^\prime$ be the complete set of sub-district colliders consisting of the farthest sub-district colliders away from $Y$ on all intervened sub-districts. Let $N^{int}(C^\prime)$ be the intervention set determined by $C^\prime$, and
$$
N^{-int}(C^\prime) = \{1,\cdots,p\} \backslash N^{int}(C^\prime).
$$
Then $N^{-int}(C^\prime)$ is an intervention-stable set with respect to all the interventions. 
\end{corollary}

\begin{proof}
It is just a consequence of Theorem \ref{theorem 1}.
\end{proof}

By choosing the farthest sub-district colliders on each intervened sub-districts, we can construct a specific intervention-stable set $N^{-int}(C^\prime)$ and stable frontier $SF_I(Y, C^\prime)$. However, it does not say that this choice is the only complete set of sub-district colliders that can lead to $SF_I(Y, C^\prime)$. In general, for two different complete sets of sub-district colliders $C^1$ and $C^2$, although their induced $N^{-int}(C^1)$ and $N^{-int}(C^2)$ may be different, the stable frontiers $SF_I(Y, C^1)$ and $SF_I(Y, C^2)$ can be the same. Conversely, if $N^{-int}(C^1)$ and $N^{-int}(C^2)$ are the same, then $SF_I(Y, C^1)$ and $SF_I(Y, C^2)$ are consequently identical. We need more assumptions to ensure that the stable frontier is both the most predictive and unique.

\begin{theorem}
\label{theorem 2}
Given the graphical models in Setting \ref{Setting 2}, if one of the two assumptions is satisfied, 
\begin{enumerate}[(1)]
    \item there is no intervened sub-district, or
    \item for every pair of sub-district colliders on each intervened sub-district, there is an intervened sub-district that has only one sub-district collider. And the only sub-district collider is an ancestor of the pair of sub-district colliders.
\end{enumerate}

\noindent Then the $N^{int}(C)$ will remain the same for different complete sets of sub-district colliders $C$. Thus the stable frontier $SF_I(Y, C)$ is unique. 
\end{theorem}




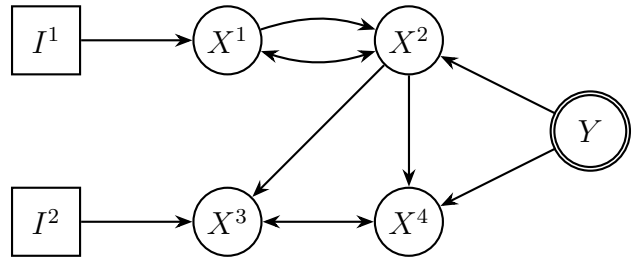
\begin{figure}[H]
\centering
\begin{tikzpicture}[
    >=Stealth,
    thick,
    scale=1.2,
    every node/.style={font=\normalsize},
    obsnode/.style={
        circle,
        draw=black,
        fill=white,
        minimum size=9mm,
        inner sep=0pt
    },
    intnode/.style={
        rectangle,
        draw=black,
        fill=white,
        minimum width=9mm,
        minimum height=9mm,
        inner sep=0pt
    },
    ynode/.style={
        double,
        circle,
        draw=black,
        fill=white,
        minimum size=10mm,
        inner sep=0pt,
        line width=0.8pt
    }
]

    \node[intnode] (i1) at (0,2) {$I^1$};
    \node[intnode] (i2) at (0,0) {$I^2$};

    \node[obsnode] (x1) at (2,2) {$X^1$};
    \node[obsnode] (x2) at (4,2) {$X^2$};
    \node[obsnode] (x3) at (2,0) {$X^3$};
    \node[obsnode] (x4) at (4,0) {$X^4$};

    \node[ynode]   (y)  at (6,1) {$Y$};

    \draw[->] (i1) -- (x1);
    \draw[->, bend left=18] (x1) to (x2);
    \draw[->] (i2) -- (x3);
    \draw[->] (y) -- (x2);
    \draw[->] (y) -- (x4);
    \draw[->] (x2) -- (x3);
    \draw[->] (x2) -- (x4);

    \draw[<->, bend right=18] (x1) to (x2);
    \draw[<->] (x3) -- (x4);

\end{tikzpicture}
\caption{Graphical illustration for condition (2) in Theorem \ref{theorem 2}.}
\label{fig:M9}
\end{figure}

Indeed, if all intervened sub-districts have at least two sub-district colliders, we choose one of the sub-district colliders $X^i$ that is closest to the intervention side. Then on the intervened sub-district where the $X^i$ locates, there is another sub-district collider $X^j$. It can be shown that $N^{-int}(C^1)$ and $N^{-int}(C^2)$ are different, where $C^1$ and $C^2$ are complete sets of sub-district colliders such that $i \in C^1, i \notin C^2$, and $j \in C^2, j \notin C^1$, because $N^{-int}(C^1)$ can not contain $X^i$ but $N^{-int}(C^2)$ contains $X^i$. However, the stable frontiers may be the same. Thus, the assumptions in Theorem \ref{theorem 2} are strong enough for the same intervention-stable set $N^{int}(C)$ but can be loosened for the unique stable frontier.

\begin{proof}
Firstly, if there is no intervened sub-district in the graphical model, then the $N^{int}(C)$ is empty; thus, the stable frontier is just the Markov blanket of $Y$ with respect to the original graph. 

Otherwise, assume the second assumption holds. Given a complete set of sub-district colliders $C$, $N^{int}(C)$ consists of sub-district colliders in $C$ and their descendants. Let $X^i$ and $X^j$ be two sub-district colliders on an intervened sub-district. There are four cases according to the relation of $i$, $j$, and $C$: $(i \in C, j \in C)$, $(i \in C, j \notin C)$, $(i \notin C, j \in C)$, and $(i \notin C, j \notin C)$. According to the assumptions, there is an intervened sub-district with only one sub-district collider $X^k$ on it. From the construction of $C$, $k \in C$, which means that $de(X^k) \subseteq N^{int}(C)$. Since $k \in an(X^i)$ and $k \in an(X^j)$, it holds that $i, j \in N^{int}(C)$. The sub-district collider $X_k$ can control the influence of the above four cases. Therefore, the $N^{int}(C)$ is the same, and the stable frontier induced from it is the same regardless of the complete sets of sub-district colliders $C$.
\end{proof}

If there is only one stable frontier in the graphical models of Setting \ref{Setting 2}, then we call it the stable blanket of $Y$, denoted by $SB_I(Y)$. The stable blanket of $Y$ is intervention-stable, which means the dependence between $Y$ and $SB_I(Y)$ can be stable in the sense that it is invariant even in new environments. Its uniqueness is important for potential identification from data.

\section{SCMs with Cycles}




In this section, we consider graphical models with causal cycles but without hidden variables.  Causal cycles create additional difficulties for identifying informative predictor sets. 
In an acyclic fully observed model, the Markov blanket and stable blanket can be characterized in terms of local graphical relations around the response. In a cyclic model, however, variables that are not locally adjacent to the response may still carry information about it through feedback relationships. In particular, variables belonging to the same SCC can play similar predictive roles, even if some of them are graphically far from $Y$ along the cycle.

The usual Markov properties based on $d$-separation and $m$-separation are not suitable for graphical models with cycles. Instead, we use $\sigma$-separation, which provides an appropriate separation criterion for cyclic graphical models under suitable solvability assumptions. Compared with the hidden-variable case, where additional predictors enter through bidirected edges in an ADMG, the cyclic case requires predictor sets to account for SCCs. We therefore adapt the notation and constructions from Section~3 by replacing individual nodes with SCCs where appropriate.

Throughout this section, we assume that there are no hidden variables. Hence an SCM induces a DG, possibly with directed cycles. Since cycles may prevent the structural assignments from having a well-defined solution, we impose a solvability condition on the relevant SCCs in the setting below.

\begin{setting}
\label{Setting 3}
Let $X \in \mathcal{X} = \mathcal{X}^1 \times \cdots \times \mathcal{X}^p$ be observable predictors, $Y \in \mathbb{R}$ be a response variable, and $I = (I^1,\cdots,I^m) \in \mathcal{I} = \mathcal{I}^1 \times \cdots \times \mathcal{I}^m$ be intervention variables which are used to formalize the interventions and act on observed random variables. Assume there is an SCM $\mathcal{S}$ over $(I,X,Y)$ such that the induced $\mathcal{G}(\mathcal{S})$ is a DG that contains cycles. The interventions $I$ are source nodes in the $\mathcal{G}(\mathcal{S})$ but do not have directed edges to the SCC of $Y$. An intervention environment $e$ corresponds to an intervention SCM $\mathcal{S}_e$ over $(I_e,X_e,Y_e)$ where the DG $\mathcal{G}(S_e)$ induced by $S_e$ does not change as the environment changes (i.e., $\mathcal{G}(\mathcal{S}_e) = \mathcal{G}(\mathcal{S})$). Assume that these SCMs are uniquely solvable with respect to each SCC of $\mathcal{G}(S_e)$.
\end{setting}

Under the solvability conditions (\cite{Bongers21}), the distribution embedded in the graphical model where the DG has cycles has the Markov property based on the $\sigma$-separation criterion. So the new Markov property provides an intuitive way to read conditional independence from the graph.

\subsection{Markov Blanket in DGs}

We can use the Markov property based on the $\sigma$-separation to define the Markov blanket in DGs.

\begin{definition}[Markov Blanket in Directed Graphs]
Given the graphical models of Setting \ref{Setting 3}, the Markov blanket of $Y$ can be defined as the smallest set $S \subseteq \{1,\cdots,p\}$, denoted by $MB(Y)$, such that 
$$
\forall j \in \{1,\cdots,p\} \backslash S,\ X^j \indep_\sigma Y | X^S.
$$
\end{definition}

So the Markov blanket can $\sigma$-separate $Y$ and remaining predictors in cyclic DGs. By Markov property, it has the best predictive performance as well as the least predictors. As the SCC replaces the node as the fundamental element in cyclic DGs, it is necessary to denote the paths in the form of SCCs. In this way, the decomposition of the Markov blanket enables SCCs to be used during classification.

We say a path $P$ in a DG has a shape 
$$
scc(X^{i_1}) - scc(X^{i_2}) - \cdots - scc(X^{i_n}),
$$
where $scc(X^{i_j}) \neq scc(X^{i_{j+1}})$, for $j = 1,\cdots,n-1$, if 
\begin{enumerate}[(1)]
    \item $P$ can be divided into $n$ pieces $(X^{(j,1)}, \cdots, X^{(j,l_j)})$, where $(j,\cdot) \in scc(X^{i_j})$ and $l_j \geq 1$, for $j = 1,\cdots,n$;
    \item and $X^{(j,l_j)}$ and $X^{(j+1,1)}$ are connected by $-$ between $scc(X^{i_j})$ and $scc(X^{i_{j+1}})$, where $-$ is either $\rightarrow$ or $\leftarrow$, for $j = 1,\cdots,n-1$.
\end{enumerate}

Additionally, the SCCs in the shapes of paths can be substituted by nodes in order to generalize the expression. In addition, a path $P$ is said to be in a SCC if all nodes on $P$ are in the SCC.

\begin{proposition}
\label{Markov blanket with cycles}
We can characterize the Markov blanket of $Y$ in DGs as follows. 
$$
MB(Y) = pa(scc(Y)) \cup scc(Y) \cup scc(ch(Y)) \cup pa(scc(ch(Y)))\ \backslash\ {Y}.
$$
\end{proposition}

\begin{proof}
Let $S = pa(scc(Y)) \cup scc(Y) \cup scc(ch(Y)) \cup pa(scc(ch(Y)))\ \backslash\ {Y}$. First, we show that $X^S$ can $\sigma$-separate $Y$ and $X^j$, for $j \in \{1,\cdots,p\} \backslash S$. Fix an $j \in \{1,\cdots,p\} \backslash S$, and let $P$ be a path between $X^j$ and $Y$. Then $P$ has the following shapes.

\begin{enumerate}[(i)]
    \item $X^j \cdots X^k \rightarrow scc(Y)$, 
    
    where $k \notin scc(Y)$. In this case, $j \neq k$, so $X^k$ can act as a non-collider on $P$, and it satisfies the third condition of $\sigma$-block. Thus the path $P$ is $\sigma$-blocked by $X^S$ as $k \in pa(scc(Y)) \subseteq S$.
    
    \item $X^j \cdots X^k \leftarrow scc(Y) - Y$, 
    
    where $k \notin scc(Y)$, and  $-$ is $\rightarrow$ or $\leftarrow$. $scc(Y) - Y$ means that there is a piece of $P$ as $(X^{i_n},e_1,X^{i_{n-1}}, \cdots, X^{i_1})$ in $scc(Y)$ such that $X^k \leftarrow X^{i_n}$ and $X^{i_1} - Y$. Hence the node $X^{i_n}$ can act as a non-collider which points out to $X^k$, a node not in the same SCC as $X^{i_n}$. Thus the path $P$ is $\sigma$-blocked by $X^S$ since $i_n \in scc(Y) \subseteq S$.
    
    \item $X^j \cdots X^k \leftarrow scc(X^i) \leftarrow Y$, 
    
    where $k \notin scc(X^i)$, $i \in ch(Y)$, and $i \notin scc(Y)$. It is possible that $j = k$. And there is a piece of $P$ as $(X^{i_n},e_1,X^{i_{n-1}}, \cdots, X^i)$ in $scc(X^i)$ such that $X^k \leftarrow X^{i_n}$ and $X^i \leftarrow Y$. So the node $X^{i_n}$ can act as a non-collider on the path pointing out to a neighboring SCC. Thus $P$ is $\sigma$-blocked by $X^S$ as $i_n \in scc(ch(Y)) \subseteq S$.
    
    \item $X^j \cdots X^k \rightarrow scc(X^i) \leftarrow Y$, 
    
    where $k \notin scc(X^i)$, $i \in ch(Y)$, and $i \notin scc(Y)$. Then $X^k$ is such a non-collider on $P$ that points to another SCC. Because $k \in pa(scc(ch(Y))) \subseteq S$, the path $P$ is $\sigma$-blocked by $X^S$.
\end{enumerate}

We can also show that $S$ is the smallest set such that it can $\sigma$-separated $Y$ and other variables, namely, MB(Y).

For any $k \in scc(Y)$, there is a directed path from $X^k$ to $Y$ as 
$$
X^k \rightarrow \cdots \rightarrow Y
$$
in $scc(Y)$. Since no collider exists on the directed path and all nodes on the path are in the same SCC, the path can not be $\sigma$-blocked by any subset of $\{1,\cdots,p\}$. Thus $scc(Y) \subseteq MB(Y)$.

For any $k \in pa(scc(Y)) \backslash scc(Y)$, there is is path as 
$$
X^k \rightarrow X^{i_n} \rightarrow \cdots \rightarrow X^{i_1} \rightarrow Y,
$$
where $i_1,\cdots,i_n \in scc(Y)$. There are no colliders on the path. Although $X^k$ points out to a node that is not in the same SCC, $X^k$ is an endpoint of the path. So this path can not be $\sigma$-blocked by any subset of $\{1,\cdots,p\}$, which means $pa(scc(Y)) \subseteq MB(Y)$.

For any $k \in scc(ch(Y)) \backslash scc(Y)$, there is a path 
$$
X^k \leftarrow X^{i_n} \leftarrow \cdots \leftarrow X^{i_1} \leftarrow Y,
$$
where $k, i_1, \cdots, i_n \in scc(X^{i_1})$ and $i_1 \notin scc(Y)$. Firstly, there is no collider on the path. Secondly, $Y$ is an end node of the path. Thus this path can not be $\sigma$-blocked by any subset of $\{1,\cdots,p\}$, which means $scc(ch(Y)) \subseteq MB(Y)$.

For any $k \in pa(scc(ch(Y))) \backslash pa(scc(Y)) \backslash scc(ch(Y))$, there is a path 
$$
X^k \rightarrow X^{i_n} \leftarrow \cdots \leftarrow X^{i_1} \leftarrow Y,
$$
where $i_1, \cdots, i_n \in scc(X^{i_1})$, $k \notin scc(X^{i_1})$, and $i_1 \notin scc(Y)$. While $X^{i_n}$ is a collider on the path, it is in the Markov blanket of $Y$ as $i_n \in scc(ch(Y))$. Besides, the nodes which point out to other SCCs are $X^k$ and $Y$, but they are end nodes of the path. Thus this path can not be $\sigma$-blocked by the Markov blanket of $Y$, which means $pa(scc(ch(Y))) \subseteq MB(Y)$.

Therefore, it has $S \subseteq MB(Y)$. Combined with the former result, $S$ is the graphical characterization of Markov blanket $MB(Y)$. 
\end{proof}

\subsection{Stable Blanket in DGs}

Then we can define the stable blanket under intervention situations. In this section, we overuse the notations of intervention-stable and intervention set in Section 3.

\begin{definition}[Intervention Set]
Given the graphical models of Setting \ref{Setting 3}, we define the intervention set as
$$
N^{int} = \{k\ | \exists i,\ i \in ch(Y) \backslash scc(Y),\ s.t.\ k \in de(scc^I(X^i))\ or\ k \in scc^I(X^i)\},
$$
where $scc^I(X^i)$ is the SCC which includes $X^i$ and has at least one element being directly intervened.
\end{definition}

\begin{theorem}
\label{theorem 3}
Given the graphical models of Setting \ref{Setting 3}, if $N^{int}$ is the intervention set, then its complementary set $N^{-int} = \{1,\cdots,p\} \backslash N^{int}$ is intervention-stable.
\end{theorem}

\begin{proof}
We can show that $N^{-int}$ is intervention-stable by proving any path between each intervention $I^l, l = 1,\cdots,m$ and $Y$ can be $\sigma$-blocked by $X^{N^{-int}}$. Let $P$ be a path connecting $I^l$ and $Y$. Since no interventions act on $scc(Y)$, $P$ has one of the following five shapes.

\begin{enumerate}[(i)]
    \item $I^l \cdots X^k \rightarrow scc(Y)$, 
    
    where $k \notin scc(Y)$. If $k \in N^{int}$, there is an $i \in ch(Y)$ such that $scc(X^i)$ is directly intervened. And there is a directed path from $Y$ to $X^k$ such that $Y \rightarrow X^i \rightarrow \cdots \rightarrow X^k$. But it still has $X^k \rightarrow scc(Y)$, leading to $scc(X^i) = scc(Y)$. This implies the SCC of $Y$ is directly intervened, contradicting assumptions in Setting \ref{Setting 3}. Additionally, it also indicates $k \in scc(Y)$, which also results in a contradiction. Thus $k \in N^{-int}$. So $P$ is $\sigma$-blocked by $X^{N^{-int}}$ as $X^k$ is a non-collider on $P$ that has an arrow heading to another SCC.
    
    \item $I^l \cdots X^k \leftarrow scc(Y) - Y$, 
    
    where $k \notin scc(Y)$, and $-$ is $\leftarrow$ or $\rightarrow$. $scc(Y) - Y$means that there is a piece of $P$ as $(X^{i_n},e_1,X^{i_{n-1}}, \cdots, X^{i_1})$ in $scc(Y)$ such that $X^k \leftarrow X^{i_n}$ and $X^{i_1} - Y$. If $scc(Y) \cap N^{int} \neq \emptyset$, let $j$ be an element of this intersection. Then there is an $i \in ch(Y)$ such that $scc(X^i)$ is directly intervened. And there exists a directed path from $Y$ to $X^j$ such that $Y \rightarrow X^i \rightarrow \cdots \rightarrow X^j$ as $j \in N^{int}$. And there is also a directed path from $X^j$ to $Y$ such that $X^j \rightarrow \cdots \rightarrow Y$ as $j \in scc(Y)$. It can induce that $scc(Y) = scc(X^i)$, so an intervention directly intervenes on $scc(Y)$, contradicting Setting \ref{Setting 3}. Hence $scc(Y) \cap N^{int} = \emptyset$, and $i_n \in N^{-int}$. Thus $P$ is also $\sigma$-blocked by $X^{N^{-int}}$ since $X^{i_1}$ is a non-collider on $P$ pointing out to another SCC.
    
    \item $I^l \rightarrow scc(X^i) \leftarrow Y$,
    
    where $i \in ch(Y)$ but $i \notin scc(Y)$. There must be a collider on the path in $scc(X^i)$, and the collider and all its descendants are not in $N^{-int}$. Thus $P$ is also $\sigma$-blocked by $X^{N^{-int}}$.
    
    \item $I^l \rightarrow \cdots X^k \leftarrow scc(X^i) \leftarrow Y$, 
    
    where $k \notin scc(X^i)$, $i \in ch(Y)$, and $i \notin scc(Y)$. And there is a piece of $P$ as $(X^{i_n},e_1,X^{i_{n-1}}, \cdots, X^i)$ in $scc(X^i)$ such that $X^k \leftarrow X^{i_n}$ and $X^i \leftarrow Y$. If $scc(X^i) \cap N^{int} \neq \emptyset$, then $scc(X^i) \subseteq N^{int}$ and $k \in N^{int}$, since all nodes in $scc(X^i) \backslash N^{int}$ and $X^k$ are descendants of elements in the intersection. Besides, there must be a collider on the path between $I^l$ and $X^{i_n}$ as there is $I \rightarrow$ and $\leftarrow X^{i_n}$ on $P$. Let $X^q$ be the closest collider to $X^{i_n}$ on $P$ between $I^l$ and $X^{i_n}$. Then $q \in de(scc(X^i))$, $q \in N^{int}$, and $de(X^q) \subseteq N^{int}$. So $P$ is $\sigma$-blocked by $X^{N^{-int}}$ as $X^q$ is a collider on the path. Otherwise, if $scc(X^i) \cap N^{int} = \emptyset$, then $P$ is also $\sigma$-blocked by $X^{N^{-int}}$ because $i_n \in N^{-int}$ and $X^{i_n}$ has an arrow heading to another SCC. 
    
    \item $I^l \rightarrow \cdots X^k \rightarrow scc(X^i) \leftarrow Y$, 
    
    where $k \notin scc(X^i)$, $i \in ch(Y)$, and $i \notin scc(Y)$. If $k \in N^{-int}$, $P$ is $\sigma$-blocked by $X^{N^{-int}}$ as $X^k$ is a non-collider on $P$ pointing out to another SCC. In the other case, $k \in N^{int}$, thus the downstream strongly connect component $scc(X^i)$ is also included in $N^{int}$. Moreover, there must be a collider on $P$ in $scc(X^i)$ as there is $X^k \rightarrow$ and $\leftarrow Y$. The collider and all its descendants are in $N^{int}$. Thus $P$ is $\sigma$-blocked by $X^{N^{-int}}$.
\end{enumerate}

Therefore, any path between $I^l$ and $Y$ is $\sigma$-blocked by $X^{N^{-int}}$, which means $I^l \indep_\sigma Y | X^{N^{-int}}$, for $l = 1, \cdots, m$. As a result, $N^{-int}$ is intervention-stable with respect to all interventions.
\end{proof}

Figure \ref{fig:M4} gives an example of a Markov blanket in cyclic cases, where the Markov blanket obviously consists of more predictors than that in acyclic situations. The reason is that the reciprocal causal relationships dilute the direct causal relationships among variables on the same cycle. The construction of a stable blanket sees below.

\begin{example}
Assume SCMs over $(X^1,\cdots,X^{10},Y,I^1,I^2)$.
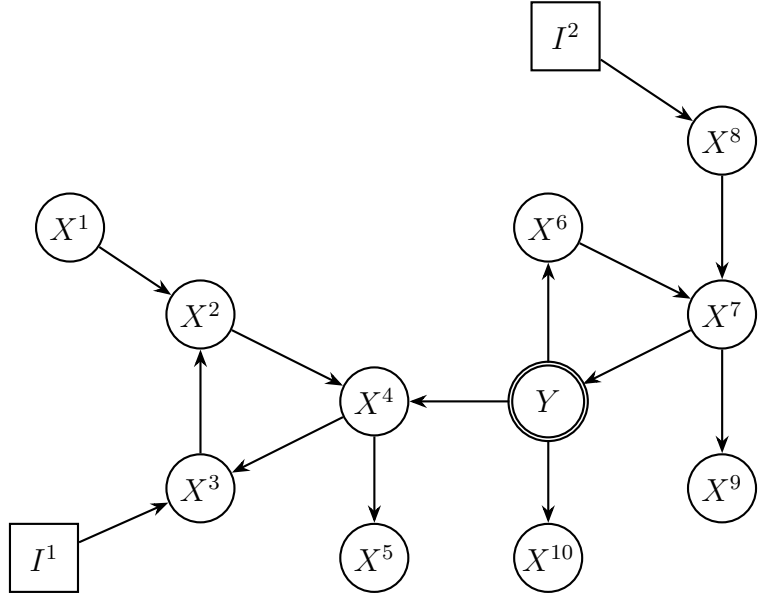
\begin{figure}[ht]
\centering
\begin{tikzpicture}[
    >=Stealth,
    thick,
    scale=1.15,
    every node/.style={font=\normalsize},
    obsnode/.style={
        circle,
        draw=black,
        fill=white,
        minimum size=9mm,
        inner sep=0pt
    },
    intnode/.style={
        rectangle,
        draw=black,
        fill=white,
        minimum width=9mm,
        minimum height=9mm,
        inner sep=0pt
    },
    ynode/.style={
        double,
        circle,
        draw=black,
        fill=white,
        minimum size=10mm,
        inner sep=0pt,
        line width=0.8pt
    }
]

    \node[obsnode] (x2) at (2,2) {$X^2$};
    \node[obsnode] (x3) at (2,0) {$X^3$};
    \node[obsnode] (x4) at (4,1) {$X^4$};

    \node[obsnode] (x1) at (0.5,3) {$X^1$};
    \node[intnode] (i1) at (0.2,-0.8) {$I^1$};

    \node[ynode]   (y)  at (6,1) {$Y$};
    \node[obsnode] (x6) at (6,3) {$X^6$};
    \node[obsnode] (x7) at (8,2) {$X^7$};

    \node[obsnode] (x5)  at (4,-0.8) {$X^5$};      
    \node[obsnode] (x8)  at (8,4) {$X^8$};
    \node[intnode] (i2)  at (6.2,5.2) {$I^2$};
    \node[obsnode] (x9)  at (8,0) {$X^9$};         
    \node[obsnode] (x10) at (6,-0.8) {$X^{10}$};   

    \draw[->] (x1) -- (x2);
    \draw[->] (x3) -- (x2);
    \draw[->] (x2) -- (x4);
    \draw[->] (x4) -- (x3);
    \draw[->] (i1) -- (x3);

    \draw[->] (y) -- (x4);
    \draw[->] (x4) -- (x5);
    \draw[->] (y) -- (x10);

    \draw[->] (y) -- (x6);
    \draw[->] (x6) -- (x7);
    \draw[->] (x7) -- (y);

    \draw[->] (x8) -- (x7);
    \draw[->] (i2) -- (x8);
    \draw[->] (x7) -- (x9);

\end{tikzpicture}
\caption{Graphical model with cycles but without hidden variables. $I^1$ and $I^2$ are interventions. The Markov blanket of $Y$ is $\{X^1,X^2,X^3,X^4,X^6,X^7,X^8,X^{10}\}$, and the stable blanket of $Y$ is $\{X^6,X^7,X^8\}$ since $scc(X^2)$ is directly intervened on.}
\label{fig:M4}
\end{figure}
\end{example}

\begin{definition}[Stable Blanket in DGs]
Given the graphical models of Setting \ref{Setting 3}, $N^{int}$ is the intervention set, then the stable blanket of $Y$ can be defined as the smallest subset $S \subseteq N^{-int} = \{1,\cdots,p\} \backslash N^{int}$ such that
$$
\forall j \in N^{-int} \backslash S :\ X^j \indep_{\sigma} Y | X^S.
$$
\end{definition}

\begin{theorem}
Given the graphical models of Setting \ref{Setting 3} and denote the induced DG by $\mathcal{G}$. Then the stable blanket consists of $Y$'s SCC, the parents of $Y$'s SCC, the SCCs of $Y$'s children, and the parents of SCCs of $Y$'s children in the sub-graph of $\mathcal{G}$ over $(N^{-int},Y,I)$, denoted by $\mathcal{G}^{(N^{-int},Y,I)}$. Moreover, it is intervention-stable.
\end{theorem}

Like the Setting \ref{Setting 1} without hidden variables and cycles, we continue to use $SB_I(Y)$ to denote the stable blanket of $Y$. Moreover, the subscripts of the sub-graph are added to distinguish the relationships in different graphs.

\begin{proof}
First, we prove the characterization of the stable blanket in the sub-graph $\mathcal{G}^{N^{-int}}$. As there is no intervention directly acting on $scc(Y)$, we can infer that the SCC of $Y$ and the parents of elements in it will reserve in the sub-graph $\mathcal{G}^{N^{-int}}$. Let $S = scc(Y) \cup pa(scc(Y)) \cup scc(ch(Y)) \cup pa(scc(ch(Y)))$ in $\mathcal{G}^{N^{-int}}$. For any $j \in N^{-int} \backslash S$, we show that every path connecting $X^j$ and $Y$ can be $\sigma$-blocked by $S$. Assume $P$ is such a path; it has one of the following shapes:

\begin{enumerate}[(i)]
    \item $X^j \cdots X^k \rightarrow scc(Y)$,
    
    where $k \in pa(scc(Y)) \backslash scc(Y)$. As $k \in S$ and $k \notin scc(Y)$, $P$ is $\sigma$-blcoked by $S$.
    
    \item $X^j \cdots X^k \leftarrow scc(Y) - Y$,
    
    where $k \notin scc(Y)$ and $-$ is $\rightarrow$ or $\leftarrow$. In this case, $P$ go via nodes in $scc(Y)$ before entering $Y$. Assuming $X^l$ is the first node $P$ has gone through in $scc(Y)$, then $X^l$ is a non-collider on it. Moreover, $l \in S$ and $l \notin scc(X^k)$, so $P$ is $\sigma$-blocked by $S$  
    
    \item $X^j \cdots X^k \leftarrow scc(X^i) \leftarrow Y$,
    
    where $k \notin scc(X^i)$, $i \in ch(Y)$, and $i \notin scc(Y)$. If $i \in N^{-int}$, let the part inside $scc(X^i)$ of $P$ be 
    $$ \leftarrow X^{i_n} - \cdots - X^{i_1} \leftarrow.$$
    Then $X^{i_n}$ is a non-collider on $P$ pointing to another SCC. We also know that $i_n \in scc(X^i) \subseteq S$. Thus $P$ is $\sigma$-blocked by $S$.
    
    Otherwise, when $i \notin N^{-int}$, then $k \notin N^{-int}$. Let $X^q$ be the farthest descendant of $X^k$ along the path $P$. It holds that $q \neq j$ because $j \in N^{-int}$. Besides, $X^q$ is a collider on $P$ since its farthest property. So $P$ is $\sigma$-blocked by $S$.
    
    \item $X^j \cdots X^k \rightarrow scc(X^i) \leftarrow Y$,
    
    where $k \notin scc(X^i)$, $i \in ch(Y)$, and $i \notin scc(Y)$. If $i \in N^{-int}$, then $k \in N^{-int}$. Furthermore, $X^k$ is a non-collider on $P$, pointing to another SCC. So $P$ is $\sigma$-blocked by $S$ as $k \in pa(scc(ch(Y))) \subseteq S$.
    
    If $i \notin N^{-int}$, then $scc(X^i) \cap N^{-int} = \emptyset$. We can know that a collider of $P$ exists inside $scc(X^i)$ as both $X^k$ and $Y$ have arrows into the SCC. Let $X^{i_q}$ be the collider. $i_q \notin S$ and $de(X^{i_q}) \cap S = \emptyset$, so $P$ is still $\sigma$-blocked by $S$.
\end{enumerate}

So far we have proved that $S$ can $\sigma$-separate $Y$ and other nodes in $N^{-int}$. In order to show that $S$ satisfies the minimum condition, the proof is similar to the second part of the proof in Proposition \ref{Markov blanket with cycles}. No matter what the set is, it can not $\sigma$-separate $Y$ and nodes that are in the four components of $S$. Thus, all $scc(Y)$, $pa(scc(Y))$, $scc(ch(Y))$, and $pa(scc(ch(Y)))$ in $\mathcal{G}^{N^{-int}}$ are subsets of the stable blanket $SB_I(Y)$. Until now, we have proved that the stable blanket of $Y$ can be characterized by the four kinds of variables.

To show $S$ can resist interventions' variation, it suffices to prove that $S$ is intervention-stable. To this end, we will prove for all $l \in \{1,\cdots,m\}$, $I^l$ and $Y$ are $\sigma$-separated by $X^{SB_I(Y)}$. We can also get the characterization from the ways paths enter $Y$. Let $P$ be a path connecting $I^l$ and $Y$. $P$ have one of the following five shapes:

\begin{enumerate}[(i)]
    \item $I^l \cdots X^k \rightarrow scc(Y)$, 
    
    where $k \notin scc(Y)$. From the proof of Theorem \ref{theorem 3}, it holds that $k \notin N^{int}$ and $scc(Y) \cap N^{int} = \emptyset$. Then $k \in pa_{\mathcal{G}^{N^{-int}}}(scc_{\mathcal{G}^{N^{-int}}}(Y))$ and $scc_{\mathcal{G}^{N^{-int}}}(Y) = scc(Y)$. Thus $k \in SB_I(Y)$, which means that $P$ is $\sigma$-blocked by $X^{SB_I(Y)}$ as $X^k$ is a non-collider on $P$ which points to a different SCC.
    
    \item $I^l \cdots X^k \leftarrow scc(Y) - Y$, 
    
    where $k \notin scc(Y)$, and $-$ is $\leftarrow$ or $\rightarrow$. $scc(Y) - Y$ means that there is a piece of $P$ as $(X^{i_n},e_1,X^{i_{n-1}}, \cdots, X^{i_1})$ in $scc(Y)$ such that $X^k \leftarrow X^{i_n}$ and $X^{i_1} - Y$. As $scc(Y) \subseteq N^{-int}$, it has $scc(Y) \subseteq SB_I(Y)$. So $P$ is $\sigma$-blocked by $X^{SB_I(Y)}$ because $i_n \in SB_I(Y)$ and $X^{i_n}$ can act a non-collider on $P$ which satisfies the third condition of $\sigma$-block.
    
    \item $I^l \rightarrow scc(X^i) \leftarrow Y$, 
    
    where $i \in ch(Y)$ but $i \notin scc(Y)$. In this case, $scc(X^i) \subseteq N^{int}$ and $de(scc(X^i)) \subseteq N^{int}$. Moreover, there must be a collider on the path in $scc(X^i)$ as there is $I^l \rightarrow$ and $\leftarrow Y$. Let $X^q$ be the collider. Then $q \notin N^{-int}$ and $de(X^q) \cap N^{-int} = \emptyset$. Thus $P$ is $\sigma$-blocked by $X^{SB_I(Y)}$ as $SB_I(Y)$ is a subset of $N^{-int}$.
    
    \item $I^l \rightarrow \cdots X^k \leftarrow scc(X^i) \leftarrow Y$, 
    
    where $k \notin scc(X^i)$, $i \in ch(Y)$, and $i \notin scc(Y)$. And there is a piece of $P$ as $(X^{i_n},e_1,X^{i_{n-1}}, \cdots, X^i)$ in $scc(X^i)$ such that $X^k \leftarrow X^{i_n}$ and $X^i \leftarrow Y$. If $scc(X^i) \subseteq N^{-int}$, then $scc(X^i) \subseteq SB_I(Y)$ as $scc(X^i) = scc_{\mathcal{G}^{N^{-int}}}(X^i)$. Then $i_n \in SB_I(Y)$ and $X^{i_n}$ points out to another strongly connect component, so $P$ is $\sigma$-blocked by $X^{SB_I(Y)}$. Otherwise, $scc(X^i) \cap N^{int} \neq \emptyset$, then $scc(X^i) \subseteq N^{int}$. And there must be a collider on $P$ between $I^l$ and $X^{i_n}$ as there is $I^l \rightarrow$ and $\leftarrow X^{i_n}$. Let $X^q$ be the closest collider to $X^{i_n}$ on $P$ between $X^{i_n}$ and $I^l$. Then $q \in de(scc(X^i))$. Hence $X^q$ and its descendants are not in $N^{-int}$, then not in $SB_I(Y)$, which means $P$ is also $\sigma$-blocked by $X^{SB_I(Y)}$ as $X^q$ is a collider on $P$.
    
    \item $I^l \rightarrow \cdots X^k \rightarrow scc(X^i) \leftarrow Y$,
    
    where $k \notin scc(X^i)$, $i \in ch(Y)$, and $i \notin scc(Y)$. If $scc(X^i) \cap N^{int} = \emptyset$, i.e., $scc(X^i) \subseteq N^{-int}$, then $k \in N^{-int}$ (otherwise that $k \in N^{int}$ will leads to $de(X^k) \subseteq N^{int}$). Thus $scc(X^i) \subseteq SB_I(Y)$ and $k \in pa_{\mathcal{G}^{N^{-int}}}(scc_{\mathcal{G}^{N^{-int}}}(X^i)) \subseteq SB_I(Y)$, inducing that $P$ is $\sigma$-blocked by $X^{SB_I(Y)}$ as $X^k$ is a non-collider on $P$ pointing out to another SCC. In the other case, $scc(X^i) \cap N^{int} \neq \emptyset$, it holds that $scc(X^i) \subseteq N^{int}$. Moreover, there must be a collider on the path in $scc(X^i)$ as there is $I^l \rightarrow$ and $\leftarrow Y$. The collider and all its descendants are also not in $N^{-int}$. Thus $P$ can be $\sigma$-blocked by $X^{SB_I(Y)}$ as the collider and its descendants are not in $SB_I(Y)$.
\end{enumerate}

Therefore $SB_I(Y)$ is intervention-stable with respect to all interventions.
\end{proof}

The stable blanket in cyclic directed graphs will reduce to the stable blanket in DAGs since every SCC will reduce to a node in DAGs.

\section{SCM with Hidden Variables and Cycles}

\subsection{Graph Structure}

The existence of cycles transfers our focus to SCCs rather than individual nodes. If there are also hidden variables, it is not enough to consider SCCs only since the relationship between different SCCs is more complicated than Setting \ref{Setting 3} in Section 4. To capture the structure accurately, we introduce the notion of the relative.

\begin{definition}[Relative]
Let $\mathcal{G} = (\mathcal{V}, \mathcal{E}, \mathcal{B})$ be a DMG. For $i, j \in \mathcal{V}$, we call $i$ a relative of $j$ if there are a set of nodes $i_1,\cdots,i_n \in \mathcal{V}$ and $i_k^1, i_k^2 \in scc(i_k),\ k = 0,\cdots,n+1$ such that
$$
scc(i_k) \neq scc(i_{k+1}),
$$
and
$$
i_k^2 \leftrightarrow i_{k+1}^1.
$$
Here we let $i_0 = i$ and $i_{n+1} = j$ for integration.
\end{definition}




\begin{example}
There is a toy example in Figure \ref{fig:M10} for the relative relationship.

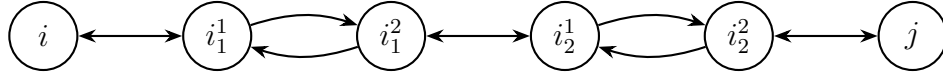
\begin{figure}[ht]
\centering
\begin{tikzpicture}[
    >=Stealth,
    thick,
    scale=1.15,
    every node/.style={font=\normalsize},
    obsnode/.style={
        circle,
        draw=black,
        fill=white,
        minimum size=9mm,
        inner sep=0pt
    }
]

    \node[obsnode] (i)  at (0,0) {$i$};
    \node[obsnode] (i1) at (2,0) {$i_1^1$};
    \node[obsnode] (i2) at (4,0) {$i_1^2$};
    \node[obsnode] (i3) at (6,0) {$i_2^1$};
    \node[obsnode] (i4) at (8,0) {$i_2^2$};
    \node[obsnode] (j)  at (10,0) {$j$};

    \draw[<->] (i) -- (i1);

    \draw[->, bend left=18]  (i1) to (i2);
    \draw[->, bend left=18] (i2) to (i1);

    \draw[<->] (i2) -- (i3);

    \draw[->, bend left=18]  (i3) to (i4);
    \draw[->, bend left=18] (i4) to (i3);

    \draw[<->] (i4) -- (j);

\end{tikzpicture}
\caption{Graphical illustration for the relative, where i is a relative of j.}
\label{fig:M10}
\end{figure}
\end{example}

We denote all relatives of $j$ as $re(j)$. Moreover, if $i$ is a relative of $j$, nodes in $i$'s SCC and district are also relatives of $j$. It can be directly induced by the definition of the relative.

\begin{proposition}
\label{disjoint}
If $i$ is a relative of $j$, then there are $i_1,\cdots,i_m \in \mathcal{V}$ such that 
\begin{enumerate}[(1)]
    \item for all $scc(i_k),\ k = 1,\cdots,m$ are disjoint;
    \item and for $scc(i_k),\ k = 0,\cdots,m$, there exist $i_k^1$ and $i_k^2$ such that $i_k^2 \leftrightarrow i_{k+1}^1$.
\end{enumerate}
\end{proposition}

\begin{proof}
Since $i$ is a relative of $j$, there are a set of nodes $i_1,\cdots,i_n \in \mathcal{V}$ and $i_k^1, i_k^2 \in scc(i_k)$ such that $scc(i_k) \neq scc(i_{k+1})$ and $i_k^2 \leftrightarrow i_{k+1}^1$, $k = 0,\cdots,n$. If $\exists\ a,b \in \{1,\cdots,n\}$ such that $scc(i_a) \cap scc(i_b) \neq \emptyset$, then $scc(i_a) = scc(i_b)$. Assuming $a < b$, we can get a new sequence $i_1, \cdots, i_a, i_{b+1}, \cdots, i_n$. The neighboring SCCs are still not identical. 

Besides, let ${i_k^1}^\prime, {i_k^2}^\prime = i_k^1, i_k^2$, when $k = 0, \cdots, a-1, b+1, \cdots, n+1$ and ${i_a^1}^\prime = i_a^1, {i_a^2}^\prime = i_b^2$. Then we have ${i_k^2}^\prime \leftrightarrow {i_{k+1}^1}^\prime$ when $k = 0, \cdots, a-2, b+1, \cdots, n$,  ${i_{a-1}^2}^\prime \leftrightarrow {i_{a}^1}^\prime$, and ${i_{a}^2}^\prime \leftrightarrow {i_{b+1}^1}^\prime$. So we can collapse the same SCCs by deleting some middle nodes. This combination operation can continue until the SCCs of these nodes are disjoint, which proves the proposition. 
\end{proof}

We introduce new notions about paths and SCCs to classify the paths and depict two neighboring SCCs.

\begin{definition}
We say two SCCs $scc(X^{i})$ and $scc(X^{j})$ are adjacent if there is a $k \in scc(X^{i})$ and $l \in scc(X^{j})$ such that $X^k$ and $X^l$ are adjacent.
\end{definition}

Given two adjacent SCCs $scc(X^{i})$ and $scc(X^{j})$, we denote that $scc(X^{i}) - scc(X^{j})$ if there exist $k \in scc(X^{i}), l \in scc(X^{j})$ such that $X^k - X^l$ where $-$ is one of three arrows $\leftarrow$, $\rightarrow$, or $\leftrightarrow$. Using this denotation, we can add multiple arrows between two adjacent SCCs. However, there are at most two arrows between these SCCs. Because three arrows must contain $\leftarrow$ and $\rightarrow$, then $scc(X^{i}) \twoset{\leftarrow}{\rightarrow} scc(X^{j})$, which means the two SCCs are indeed identical.

If $scc(X^i)$ and $scc(X^j)$ are disjoint and adjacent, we can use $scc(X^i) \twoset{-}{-} scc(X^j)$ to show the connection of these two SCCs. The $\twoset{-}{-}$ represents arrows between $scc(X^i)$ and $scc(X^j)$. There may be only one arrow. Besides, two arrows are also possible except $\twoset{\leftarrow}{\rightarrow}$. So $scc(X^i) \twoset{-}{-} scc(X^j)$ is a collection of all five situations: $scc(X^i) \leftarrow scc(X^j)$, $scc(X^i) \rightarrow scc(X^j)$, $scc(X^i) \leftrightarrow scc(X^j)$, $scc(X^i) \twoset{\rightarrow}{\leftrightarrow} scc(X^j)$, and $scc(X^i) \twoset{\leftarrow}{\leftrightarrow} scc(X^j)$.

\begin{definition}
Given two disjoint SCCs $scc(X^i)$ and $scc(X^j)$, if there exists a bidirected edge between them, then we say $scc(X^i)$ is a mate of $scc(X^j)$.
\end{definition}

Then we define $ma(X^j) = \{i\ |\ scc(X^i)\ is\ a\ mate\ of  scc(X^j)\}$.

Assume a sequence of SCCs $scc(X^{i_1}), scc(X^{i_2}), \cdots, scc(X^{i_n})$ where $scc(X^{i_j})$ and $scc(X^{i_{j+1}})$ are adjacent but not identical, for $j = 1,\cdots,n-1$. We can write as
\begin{equation}
    scc(X^{i_1}) \twoset{-}{-} scc(X^{i_2}) \twoset{-}{-} \cdots \twoset{-}{-} scc(X^{i_n}) \label{5.1}
\end{equation}
to represent those connections.

The structure allows two SCCs are identical if they are not adjacent. In addition, the SCCs can be replaced by nodes if they are not in the neighboring SCCs, which is beneficial to simplify the expression afterward.

The relative relationship extends our ability to explore the graph structure. Then we consider SCMs with hidden variables and cycles which have the Markov property with respect to DMGs under assumptions in the following setting.

\begin{setting}
\label{Setting 4}
Let $X \in \mathcal{X} = \mathcal{X}^1 \times \cdots \times \mathcal{X}^p$ be observable predictors, $H \in \mathcal{H} = \mathcal{H}^1 \times \cdots \times \mathcal{H}^q$ be hidden variables, $Y \in \mathbb{R}$ be a response variable and $I = (I^1,\cdots,I^m) \in \mathcal{I} = \mathcal{I}^1 \times \cdots \times \mathcal{I}^m$ be intervention variables which are used to formalize the interventions and act on observed random variables. Assume there exists a SCM with hidden variables $\mathcal{S}$ over $(I,X,H,Y)$ such that the $\mathcal{G}(\mathcal{S})$ is a cyclic directed graph, and $\mathcal{G}(\mathcal{S})^{(I,X,Y)}$ generated by doing latent projection on $\mathcal{G}(\mathcal{S})$ over $(I,X,Y)$ is a cyclic directed mixed graph. An intervention environment $e$ corresponds to an interventional SCM with hidden variables $\mathcal{S}_e$ over $(I_e,X_e,H_e,Y_e)$. Doing the latent projection on $\mathcal{G}(\mathcal{S}_e)$ with respect to $H_e$ will lead to a cyclic directed mixed graph $\mathcal{G}(\mathcal{S}_e)^{(I_e,X_e,Y_e)}$, which keeps fixed among environments (i.e., $\mathcal{G}(\mathcal{S})^{(I,X,Y)} = \mathcal{G}(\mathcal{S}_e)^{(I_e,X_e,Y_e)}$). Each SCM with hidden variables is uniquely solvable with respect to the set of hidden variables, and the marginalization of it is uniquely solvable with respect to every SCC in $\mathcal{G}(\mathcal{S})^{(I,X,Y)}$. Furthermore, assume that there is no intervention directly acting on the relatives of $Y$.
\end{setting}

\subsection{Markov Blanket in DMGs}

\begin{definition}[Markov Blanket in DMGs]
Given the graphical models of Setting \ref{Setting 4}, we define the Markov blanket of $Y$ as the smallest subset $S \subseteq \{1,\cdots,p\}$ such that 
$$
\forall j \in \{1,\cdots,p\} \backslash S:\ X^j \indep_{\sigma} Y | X^S.
$$
\end{definition}

\begin{proposition}
The Markov blanket in the cyclic directed mixed graph can be characterized as 
$$
MB(Y) = pa(re(Y)) \cup re(Y) \cup re(ch(Y)) \cup pa(re(ch(Y)))\ \backslash\ \{Y\}.
$$
\end{proposition}

\begin{proof}
 Let $S = pa(re(Y)) \cup re(Y) \cup re(ch(Y)) \cup pa(re(ch(Y)))\ \backslash\ \{Y\}$, and we first show that $Y$ and $X^{\{1,\cdots,p\} \backslash S}$ can be $\sigma$-separated given $X^S$. It suffices to prove that for all $j \in \{1,\cdots,p\} \backslash S$, each path between $Y$ and $X^j$ can be $\sigma$-blocked by $X^S$. Let $j \in \{1,\cdots,p\} \backslash S$ and $P$ be a path connecting $X^j$ and $Y$. $P$ must enter relatives of $Y$ first before going into $Y$, so we can divide $P$ into four types.

\begin{enumerate}[(i)]
    \item $X^j \cdots X^k \rightarrow re(Y)$,
    
    where $k \notin re(Y)$ and there is a part of $P$ as $(X^{i_n}, e_1, X^{i_{n-1}}, \cdots, Y)$ ($n \geq 0$) in $re(Y)$ such that $X^k \rightarrow X^{i_n}$. $k \notin re(Y)$ can infer that $k \notin scc(X^{i_n})$. Then $X^k$ is a non-collider on $P$ pointing to a neighboring SCC. Thus $P$ is $\sigma$-blocked by $X^S$ as $k \in pa(re(Y)) \subseteq S$.
    
    \item $X^j \cdots X^k \leftarrow re(Y) - Y$,
    
    where $k \notin re(Y)$, and $-$ is $\leftarrow$, $\rightarrow$, or $\leftrightarrow$. $re(Y) - Y$ implies that there is a part of $P$ as $(X^{i_n}, e_1, X^{i_{n-1}}, \cdots, X^{i_1})$ ($n \geq 1$) in $re(Y)$ such that $X^k \leftarrow X^{i_n}$ and $X^{i_1} - Y$. $P$ is also $\sigma$-blocked by $X^S$ since $i_n \in re(Y) \subseteq S$ and $X^{i_n}$ can act as a non-collider on $P$ satisfying the third condition of $\sigma$-block.
    
    \item $X^j \cdots X^k \leftarrow re(X^i) \leftarrow Y$,
    
    where $k \notin re(X^i)$, $i \in ch(Y)$, but $i \notin re(Y)$. So there is a part of $P$ as 
    $$
    (X^{i_n}, e_1, X^{i_{n-1}}, \cdots, X^i)
    $$
    in $re(X^i)$ such that $X^k \leftarrow X^{i_n}$ and $X^i \leftarrow Y$. We can get $k \notin scc(X^{i_n})$ as $k \notin re(X^i)$. Thus $X^{i_n}$ is a non-collider on $P$ pointing out to another SCC. Then $P$ is $\sigma$-blocked by $X^S$ as $i_n \in re(ch(Y)) \subseteq S$.
    
    \item $X^j \cdots X^k \rightarrow re(X^i) \leftarrow Y$,
    
    where $k \notin re(X^i)$, $i \in ch(Y)$, but $i \notin re(Y)$. There is a part of $P$ as
    $$
    (X^{i_n}, e_1, X^{i_{n-1}}, \cdots, X^i)
    $$
    in $re(X^i)$ such that $X^k \rightarrow X^{i_n}$ and $X^i \leftarrow Y$. $X^k$ is a non-collider on $P$ heading to another SCC ($k \notin scc(X^{i_n})$), so $P$ is $\sigma$-blocked by $X^S$ as $k \in pa(re(ch(Y))) \subseteq S$.
\end{enumerate}
Therefore, $S$ can $\sigma$-separate $Y$ and remaining variables. 

Then we prove that $X^S$ is the smallest set with the property of $\sigma$-separating $Y$ and remaining variables. We say a path is in a SCC if all its nodes are in the SCC.

Let $i \in scc(Y)$; then there is a directed path in $scc(Y)$ as 
$$
X^i \rightarrow \cdots \rightarrow Y.
$$
$X^{MB(Y)}$ can not $\sigma$-blocked this path because there is no non-collider pointing out to another SCC. Thus $scc(Y) \subseteq MB(Y)$.

Let $i \in re(Y)$, so there are a series of nodes $X^{i_1},\cdots,X^{i_n}$ and $i_j^1, i_j^2 \in scc(X^{i_j})$ such that $scc(X^{i_j})$ are disjoint (by Proposition \ref{disjoint}) and $X^{i_j^2} \leftrightarrow X^{i_{j+1}^1}$, $j = 0,\cdots,n$. Then we can get a path as 
$$
X^i = X^{i_0^1} -- X^{i_0^2} \leftrightarrow X^{i_1^1} -- X^{i_1^2} \leftrightarrow \cdots \leftrightarrow X^{i_n^1} -- X^{i_n^2} \leftrightarrow X^{i_{n+1}^1} -- X^{i_{n+1}^2} = Y
$$
where $--$ between $X^{i_j^1}$ and $X^{i_j^2}$ represents a possible directed path between $X^{i_j^1}$ and $X^{i_j^2}$ in the corresponding SCC $scc(X^{i_j})$. At first, $i_{n+1}^1 \in MB(Y)$ as it belongs to $scc(Y)$. Then $i_n^2 \in MB(Y)$ since there is a path 
$$
X^{i_n^2} \leftrightarrow X^{i_{n+1}^1} \rightarrow \cdots \rightarrow Y
$$
which can not be $\sigma$-blocked by $X^{MB(Y)}$ as no non-colliders point out to another SCC. After, we consider $X^{i_n^1}$. There is a path from $X^{i_n^1}$ to $Y$ as 
$$
X^{i_n^1} \rightarrow \cdots \rightarrow X^{i_n^2} \leftrightarrow X^{i_{n+1}^1} \rightarrow \cdots \rightarrow Y.
$$
Although $X^{i_n^2}$ is a collider on the path, $i_n^2 \in MB(Y)$. And all non-colliders on it can not point out to a different SCC. Thus this path can not be $\sigma$-blocked by $X^{MB(Y)}$. Recursively, $i_{n-1}^2, \cdots, i_0^2$ must be in $MB(Y)$. As a consequence, $i$ is in $MB(Y)$, which means that $MB(Y)$ must contain all relatives of $Y$.

A direct inference of $re(Y) \subseteq MB(Y)$ is that $pa(re(Y)) \subseteq MB(Y)$. Let $i \in pa(re(Y))$, then there is a path as
$$
X^i \rightarrow X^{i_0^1} -- X^{i_0^2} \leftrightarrow X^{i_1^1} -- X^{i_1^2} \leftrightarrow \cdots \leftrightarrow X^{i_n^1} -- X^{i_n^2} \leftrightarrow X^{i_{n+1}^1} -- X^{i_{n+1}^2} = Y
$$
where $--$ has the same representation as above, $i_j^1$ and $i_j^2$ are in the same SCC $scc(X^{i_j})$, for $j = 0, \cdots, n+1$, and these SCCs are disjoint. Because $i_j^k \in MB(Y),\ j = 0, \cdots, n+1,\ k = 1,2$ (except $Y$) and no non-endpoint non-collider points to another SCC, this path can not be $\sigma$-blocked by $X^{MB(Y)}$. So $pa(re(Y)) \subseteq MB(Y)$.

Then we look at the relatives of $Y$'s children. It is obvious that $ch(Y) \subseteq MB(Y)$. Let $i \in re(X^l)$ and $l \in ch(Y)$. There is a series of nodes $X^{i_1},\cdots,X^{i_n}$ whose SCCs are disjoint and $i_j^1, i_j^2 \in scc(X^{i_j})$ such that $X^{i_j^2} \leftrightarrow X^{i_{j+1}^1}$. It can induce a path as 
$$
X^i = X^{i_0^1} -- X^{i_0^2} \leftrightarrow X^{i_1^1} -- X^{i_1^2} \leftrightarrow \cdots\ \leftrightarrow X^{i_n^1} -- X^{i_n^2} \leftrightarrow X^{i_{n+1}^1} -- X^{i_{n+1}^2} = X^l \leftarrow Y,
$$
where $--$ represents a possible directed path in the corresponding SCC. A path 
$$
X^{i_{n+1}^1} \leftarrow \cdots \leftarrow X^l \leftarrow Y
$$
can not be $\sigma$-blocked by $X^{MB(Y)}$ since no non-endpoint non-collider on the path points to another SCC, which means $i_{n+1}^1 \in MB(Y)$. Moreover, $X^{MB(Y)}$ can not $\sigma$-block 
$$
X^{i_n^2} \leftrightarrow X^{i_{n+1}^1} \leftarrow \cdots \leftarrow X^l \leftarrow Y
$$ 
as $i_{n+1}^1 \in MB(Y)$, so $i_n^2 \in MB(Y)$. Subsequently, it has $i_n^2, \cdots, i_0^1 \in MB(Y)$. Therefore, any relative of $Y$'s children should be in the Markov blanket of $Y$, i.e., $re(ch(Y)) \subseteq MB(Y)$.

Furthermore, let $i \in pa(re(ch(Y)))$, then there is a path as 
$$
X^i \rightarrow X^{i_0^1} -- X^{i_0^2} \leftrightarrow X^{i_1^1} -- X^{i_1^2} \leftrightarrow \cdots\ \leftrightarrow X^{i_n^1} -- X^{i_n^2} \leftrightarrow X^{i_{n+1}^1} -- X^{i_{n+1}^2} = X^l \leftarrow Y.
$$
It is almost the same as the path in $i \in re(ch(Y))$ case other than one end is that $X^i \rightarrow X^{i_0^1}$. Since $re(ch(Y)) \subseteq MB(Y)$ and there are no non-endpoint non-collider points out to another SCC,  this path can not be $\sigma$-blocked by $X^{MB(Y)}$. Thus $i \in MB(Y)$, which means $pa(re(ch(Y))) \subseteq MB(Y)$.

Thus, we prove that $MB(Y)$ must contain the four parts of $S$. Combined with the property that $X^S$ can $\sigma$-separate $Y$ and other variables, $MB(Y)$ can be characterized as $S$. 
\end{proof}

\subsection{intervention-stable in DMGs}

\subsubsection{Sub-structure on Intervention}

In the Setting \ref{Setting 4}, the Markov blanket of $Y$ can $\sigma$-separate the response and other variables, so if the distribution of $(I, X, Y)$ has the Markov property based on $\sigma$-separation with respect to the DMG, there is a conditional independence as $Y \indep X^{\{1,\cdots,p\} \backslash MB(Y)} | X^{MB(Y)}$. Then it has 
$$
E(Y | X^{\{1,\cdots,p\}}) = E(Y | X^{\{1,\cdots,p\} \backslash MB(Y)}, X^{MB(Y)}) = E(Y | X^{MB(Y)}).
$$
Thus the Markov blanket of $Y$ is an appropriate set of predictors. It can induce the same results as the Markov blanket in Section 2 or Section 3 when the DMG simplifies to a simpler graph in Setting \ref{Setting 2} or Setting \ref{Setting 3}.

However, the Markov blanket in DMGs ignores the effects of interventions. If there are interventions corresponding to different environments, the intervention-stable set will contain correspondingly more elements than before. Here we overuse the definition of intervention-stable to describe the property of some sets that can prevent $Y$ from interventions. The sub-structure containing interventions is of our interest to find intervention-stable sets.

\begin{definition}[Intervened Component District]
An intervened component district consists of an intervention, a response, and some SCCs if they are connected as
$$
I \rightarrow scc(X^{i_n}) \twoset{-}{\leftrightarrow} scc(X^{i_{n-1}}) \twoset{-}{\leftrightarrow} \cdots \twoset{-}{\leftrightarrow} scc(X^{i_1}) \leftarrow Y,
$$
where $scc(X^{i_j}) \neq scc(X^{i_{j+1}}),\ i = 1,\cdots,n-1$.
\end{definition}

\begin{definition}[Component Collider]
\label{ComCol}
We call one SCC $scc(X^{i_j})$ in a sequence of SCCs as a component collider if the SCC has no edges heading to its preceding and succeeding SCCs, i.e.,

\begin{table}[ht]\centering
\renewcommand{\arraystretch}{1.5}
\begin{tabular}{l | l}
\hline

 $scc(X^{i_{j-1}}) \leftrightarrow scc(X^{i_j}) \leftrightarrow scc(X^{i_{j+1}})$ &  $scc(X^{i_{j-1}}) \twoset{\rightarrow}{\leftrightarrow} scc(X^{i_j}) \leftrightarrow scc(X^{i_{j+1}})$ \\  \hline

$scc(X^{i_{j-1}}) \leftrightarrow scc(X^{i_j}) \twoset{\leftarrow}{\leftrightarrow} scc(X^{i_{j+1}})$  & $scc(X^{i_{j-1}}) \twoset{\rightarrow}{\leftrightarrow} scc(X^{i_j}) \twoset{\leftarrow}{\leftrightarrow} scc(X^{i_{j+1}})$ \\ \hline

$scc(X^{i_{j-1}}) \rightarrow scc(X^{i_j}) \twoset{\leftarrow}{\leftrightarrow} scc(X^{i_{j+1}})$  & $scc(X^{i_{j-1}}) \rightarrow scc(X^{i_j}) \leftrightarrow scc(X^{i_{j+1}})$ \\ \hline

$scc(X^{(i_{j-1})}) \twoset{\rightarrow}{\leftrightarrow} scc(X^{i_j}) \leftarrow scc(X^{i_{j+1}})$  & $scc(X^{i_{j-1}}) \leftrightarrow scc(X^{i_j}) \leftarrow scc(X^{i_{j+1}}$) \\ \hline

$scc(X^{i_{j-1}}) \rightarrow scc(X^{i_j}) \leftarrow scc(X^{i_{j+1}})$  &  \\ \hline
\end{tabular}
\renewcommand{\arraystretch}{1.5}
\label{Tab:T2}
\caption{Cases of component colliders}
\end{table}
\end{definition}

The SCCs in the definition of the component collider can be replaced by nodes that are not in its neighboring SCCs. This is also beneficial to include interventions and the response. Note that the two non-adjacent SCCs in Definition \ref{ComCol} can be identical.

\begin{definition}[Intervened Component Collider]
A component collider on an intervened component district is an intervened component collider if it has no directed path into Y. 
\end{definition}

\begin{proposition}
Given the graphical models of Setting \ref{Setting 4}, there must be an intervened component collider on each intervened component district.
\end{proposition}

\begin{proof}
 Firstly, given an intervened component district
 $$
I \rightarrow scc(X^{i_n}) \twoset{-}{\leftrightarrow} scc(X^{i_{n-1}}) \twoset{-}{\leftrightarrow} \cdots \twoset{-}{\leftrightarrow} scc(X^{i_1}) \leftarrow Y,
 $$
 there must be a component collider $scc(X^{i_j})$  as there is $I \rightarrow$ and $\leftarrow Y$. If $scc(X^{i_j})$ has no directed path into $Y$, then it is an intervened component collider. 

Otherwise, assume $scc(X^{i_j})$ has a directed path into $Y$. Then if there is no component collider between $scc(X^{i_j})$ and $Y$, it can induce that
$$
Y \rightarrow scc(X^{i_1}) \twoset{\rightarrow}{\leftrightarrow} \cdots \twoset{\rightarrow}{\leftrightarrow} scc(X^{i_j}).
$$
So $i_j \in scc(Y)$ as $i_j \in an(Y)$ and $i_j \in de(Y)$. Since 
$$
scc(X^{i_n}) \twoset{\rightarrow}{\leftrightarrow} \cdots \twoset{\rightarrow}{\leftrightarrow} scc(X^{i_j}) = scc(Y),
$$
it has $scc(X^{i_1}) \subseteq re(Y)$, which is contradicting the assumption that no interventions act on the relatives of $Y$ under Setting \ref{Setting 4}. Therefore, there is a component collider $scc(X^{i_k})$ between $scc(X^{i_j})$ and $Y$ on the intervened component district.

If the $scc(X^{i_k})$ has a directed path into $Y$, there must another component collider $scc(X^{i_l})$ between $scc(X^{i_k})$ and $Y$. As the number of nodes is finite, we can always find a component collider in the intervened component district with no directed path into $Y$; in other words, it is an intervened component collider.
\end{proof}

We consider paths through SCCs.

\begin{definition}
    Given a sequence of SCCs 
    $$
    scc(X^{i_1}) \twoset{-}{-} scc(X^{i_2}) \twoset{-}{-} \cdots \twoset{-}{-} scc(X^{i_n}),
    $$
    we say a path $P$ has a shape of this sequence if $P$ can be divided into $n$ pieces $P^1, \cdots, P^n$ with $P^k$ is in $scc(X^{i_k}),\ k = 1, \cdots, n$ and $P^l - P^{l+1}$ for which $-$ is one of arrows from $scc(X^{i_l}) \twoset{-}{-} scc(X^{i_{l+1}})$, $l = 1, \cdots, n-1$.
\end{definition}

The SCCs can be substituted by nodes if they are not in the neighboring two SCCs. In this case, the path must go via these nodes.

\begin{proposition}
\label{component block}
Assume a path $P$ has a part with the shape of $(scc(X^i), scc(X^j), scc(X^k))$, 
where $i,k \notin scc(X^j)$ and $scc(X^j)$ is a component collider. Let $S$ be any set satisfying $S \cap scc(X^j) = S \cap de(scc(X^j)) = \emptyset$. Then $P$ is $\sigma$-blocked by $X^S$.
\end{proposition}

\begin{proof}
Let the part be divided into $P^i$, $P^j$, and $P^k$, which correspond to three SCCs, respectively. Then the part has one of the four cases holds: $P^i \rightarrow P^j \leftarrow P^k$, $P^i \leftrightarrow P^j \leftarrow P^k$, $P^i \rightarrow P^j \leftrightarrow P^k$, and $P^i \leftrightarrow P^j \leftrightarrow P^k$. In every case, there must be a collider $X^q$ on $P^j$. Since $q \notin S$ and $q \notin an(S)$, we can obtain that $P$ is $\sigma$-blocked by $X^S$ as $X^q$ is a collider on the whole path $P$.
\end{proof}

\subsubsection{intervention-stable Set}

\begin{definition}[Eligible Set of Components]
Given the graphical models of Setting \ref{Setting 4}, let $s^1, s^2,\cdots, s^l$ be all the intervened component districts. We define a set of SCCs $E$ as an eligible set of components if 
\begin{equation}
    E = \{c^{1,1},\cdots,c^{1,i_1},\cdots,c^{l,1},\cdots,c^{l,i_l}\} \label{eligible}
\end{equation}
where $c^{k,\cdot }$ is an intervened component collider on the intervened component district $s^k$, $i_k \geq 1,\ k = 1,\cdots,l$.
\end{definition}

\begin{definition}[Intervention Set]
Given an eligible set of components $E$ as \eqref{eligible}, a node set $N^{int}(E)$ is called an intervention set given $E$, if 
$$
N^{int}(E) = \{j \in \{1,\cdots,p\}\ |\ \exists k \in \{1,\cdots,l\}, s \in \{1,\cdots,i_k\},\ s.t.\  j \in c^{k,s}\ or\ j \in de(X^{c^{k,s}})\}. 
$$
\end{definition}

\begin{theorem}
\label{theorem 4}
Given the graphical models of Setting \ref{Setting 4}, $E$ is an eligible set of components as \eqref{eligible}. Then 
$$
N^{-int}(E) = \{1,\cdots,p\}\ \backslash\ N^{int}(E)
$$
is an intervened stable set.
\end{theorem}

\begin{proof}
We prove that $I^l \indep_{\sigma} Y\ |\ X^{N^{-int}(E)}$, $l = 1,\cdots,m$ by showing that every path between $I^l$ and $Y$ is $\sigma$-blocked by $X^{N^{-int}(E)}$. As no intervention happens on the relatives of $Y$ under Setting \ref{Setting 4}, interventions can not directly act on either $scc(dis(Y))$ or $dis(scc(Y))$. Let $P$ be a path from $I^l$ to $Y$. There must be a node on $P$ before the remaining part of $P$ is in $re(Y)$. So $P$ has all six shapes into $Y$.

\begin{enumerate}[(i)]
    \item $I^l \rightarrow \cdots X^j \rightarrow scc(X^{i_n}) \twoset{-}{\leftrightarrow} \cdots \twoset{-}{\leftrightarrow} scc(X^{i_1}) \twoset{-}{\leftrightarrow} scc(Y),$
    
    where $n \geq 0$, $j \notin ma(X^{i_n})$, and $j \notin scc(X^{i_n})$. If $j \in N^{-int}(E)$, $P$ is $\sigma$-blocked by $N^{-int}(E)$ as $X^k$ is a non-collider pointing to another SCC. If $j \in N^{int}(E)$, then $de(X^j) \subseteq N^{int}(E)$. Since there is no direct edge from $N^{int}(E)$ to $Y$, there must be a component collider between $scc(X^{i_n})$ and $scc(Y)$. Otherwise 
    $$
    X^j \rightarrow scc(X^{i_n}) \twoset{\rightarrow}{\leftrightarrow} \cdots \twoset{\rightarrow}{\leftrightarrow} scc(X^{i_1}) \twoset{\rightarrow}{\leftrightarrow} scc(Y),
    $$
    resulting in $Y \in de(X^j)$). Let $scc(X^{i_q})$ be the closest such component collider to $scc(X^{i_n})$, which means that $i_q \in de(X^j)$. Moreover, it has $scc(X^{i_q}) \cap N^{-int}(E) = de(scc(X^{i_q})) \cap N^{-int}(E) = \emptyset$. By Proposition \ref{component block}, $P$ is still $\sigma$-blocked by $N^{-int}(E)$. In particular, when $n = 0$, $j \in N^{-int}(E)$ under Setting \ref{Setting 4}. Thus $pa(scc(Y)) \subseteq N^{-int}(E)$.
    
    \item $I^l \rightarrow \cdots X^j \leftarrow scc(X^{i_n}) \twoset{-}{\leftrightarrow} \cdots \twoset{-}{\leftrightarrow} scc(X^{i_1}) \twoset{-}{\leftrightarrow} scc(Y),$ 
    
    where $n \geq 1$, $j \notin ma(X^{i_n})$, and $j \notin scc(X^{i_n})$. There is a $t \in scc(X^{i_n})$ such that $X^t \rightarrow X^j$ is a part of $P$. If $t \in N^{-int}(E)$, $P$ is $\sigma$-blocked by $N^{-int}(E)$ as $X^t$ is a non-collider on $P$ pointing to another SCC. Otherwise, $t \in N^{int}(E)$, then $de(X^t) \subseteq N^{int}(E)$. And there must be a collider on $P$ between $X^t$ and $I^l$. Let $X^q$ be the closest collider to $X^t$ on $P$ between $X^t$ and $I^l$, which means $q \in de(X^t)$. Since $X^q$ and its descendants are not in $N^{-int}(E)$, $P$ is $\sigma$-blocked by $N^{-int}(E)$.
    
    \item $I^l \rightarrow \cdots X^j \leftarrow scc(Y),$
    
    where $j \notin scc(Y)$ and $j \notin ma(Y)$. And the part of $P$ in $scc(Y)$ is $(X^{i_n}, e_1, X^{i_{n-1}}, \cdots, Y)$ such that $X^j \leftarrow X^{i_n}$. It has $i_n \in N^{-int}(E)$ as there is no direct path from $N^{int}(E)$ to $Y$. $X^{i_n}$ can act as a non-collider that points to another SCC, so $P$ is $\sigma$-blocked by $N^{-int}(E)$.
    
    \item $I^l \rightarrow scc(X^{i_n}) \twoset{-}{\leftrightarrow} \cdots  \twoset{-}{\leftrightarrow} scc(X^{i_1}) \leftarrow Y,$ 
    
    where $i_1 \notin scc(Y)$, $i_1 \in ch(Y)$, and $i_1 \notin ma(Y)$. There must be an intervened component collider $c^{i,s}$ on this intervened component district which is in $E$. Then $c^{i,s} \subseteq N^{int}(E)$ and $de(X^{c^{i,s}}) \subseteq N^{int}(E)$, which means that $P$ can be  $\sigma$-blocked by $N^{-int}(E)$ by Proposition \ref{component block}.
    
    \item $I^l \rightarrow \cdots X^j \rightarrow scc(X^{i_n}) \twoset{-}{\leftrightarrow} \cdots  \twoset{-}{\leftrightarrow} scc(X^{i_1}) \leftarrow Y,$
    
    where $j \notin ma(X^{i_n})$, $j \notin scc(X^{i_n})$ $i_1 \notin scc(Y)$, $i_1 \in ch(Y)$, and $i_1 \notin ma(Y)$. If $j \in N^{-int}(E)$, then $P$ is $\sigma$-blocked by $N^{-int}(E)$ as $X^j$ is an non-collider satisfying the third condition of $\sigma$-block. Otherwise, $j \in N^{int}(E)$, and there must be a component collider between $X^j$ and $Y$ as there is $X^j \rightarrow$ and $Y \leftarrow$. Let $scc(X^{i_q})$ be the closest such component collider to $X^j$, hence $i_q \in de(X^j)$. Then $scc(X^{i_q}) \cap N^{-int}(E) = de(scc(X^{i_q})) \cap N^{-int}(E) = \emptyset$, $P$ is $\sigma$-blocked by $N^{-int}(E)$ by Proposition \ref{component block}.
    
    \item $I^l \rightarrow \cdots X^j \leftarrow scc(X^{i_n}) \twoset{-}{\leftrightarrow} \cdots  \twoset{-}{\leftrightarrow} scc(X^{i_1}) \leftarrow Y,$ 
    
     where $j \notin ma(X^{i_n})$, $j \notin scc(X^{i_n})$, $i_1 \notin scc(Y)$, $i_1 \in ch(Y)$, and $i_1 \notin ma(Y)$. There is a $t \in scc(X^{i_n})$ such that $X^j \leftarrow X^t$ is a part of $P$. If $t \in N^{-int}(E)$, $P$ is $\sigma$-blocked by $N^{-int}(E)$ as $X^t$ is a non-collider pointing to another SCC. If $t \in N^{int}(E)$, there must be a collider on 
     $$
     I^l \rightarrow \cdots X^j \leftarrow X^t.
     $$
     Let $X^q$ be the closest collider to $X^t$ between $I^l$ and $X^t$. Then $q \notin N^{-int}(E)$ and $de(X^q) \cap N^{-int}(E) = \emptyset$ as $q \in de(X^t)$, so $P$ is still $\sigma$-blocked by $N^{-int}(E)$.
\end{enumerate}

Therefore, $N^{-int}(E)$ is an intervention-stable set with respect to all interventions. 
\end{proof}

Although $N^{-int}(E)$ is intervention-stable, which means the conditional distribution of $Y$ on it can generalize to even unobserved interventional environments, using all predictors in $N^{-int}(E)$ can still be computationally expensive. We introduce the concept of the stable frontier such that it is the smallest subset of $N^{-int}(E)$ as long as it is as informative as $N^{-int}(E)$.

\subsubsection{Stable Blanket in DMGs}

\begin{definition}[Stable Frontier in DMGs]
For each eligible set of components $E$, we can define a stable frontier denoted by $SF_I(Y, E)$ as the smallest node set $S \subseteq N^{-int}(E) = \{1,\cdots,p\} \backslash N^{int}(E)$ that satisfies
$$
\forall j \in N^{-int}(E) \backslash S :\ X^j \indep_{\sigma} Y | X^S.
$$
\end{definition}

We clarify for every eligible set of components $E$, we can identify the stable frontier determined by $E$.

\begin{theorem}
Given the graphical models of Setting \ref{Setting 4}, the DMG after latent projection is $\mathcal{G}$. The stable frontier of $Y$ given an eligible set of components $E$ consists of the relatives of Y, the parents of the relatives of Y, the relatives of Y's children, and the parents of the relatives of Y's children in the sub-graph of $\mathcal{G}$ over $(N^{-int}(E),Y,I)$. And it is intervention-stable with respect to all interventions.
\end{theorem}

\begin{proof}
Let $S = re(Y) \cup pa(re(Y)) \cup re(ch(Y)) \cup pa(re(ch(Y)))$ in the sub-graph $\mathcal{G}^{(N^{-int}(E),Y,I)}$. we aim to show that for $j \in N^{-int}(E) \backslash S :\ X^j \indep_{\sigma} Y | X^S$. Let $P$ be a path between $X^j$ and $Y$:
\begin{enumerate}[(i)]
    \item $X^j \cdots X^k \rightarrow scc(Y)$,

    where $k \notin scc(Y)$. Note here $k \neq j$; otherwise, $j \in pa(scc(Y))$. We have $k \notin N^{int}(E)$; otherwise, there will be a directed path from an intervened component collider to $Y$. Since the arrow out of $X^k$ points to a different SCC. So $P$ is $\sigma$-blocked by $S$ as $k \in pa(scc(Y)) \subset S$.

    \item $scc(X^j) \twoset{-}{\leftrightarrow} scc(X^{i_n}) \twoset{-}{\leftrightarrow} \cdots \twoset{-}{\leftrightarrow} scc(X^{i_1}) \twoset{-}{\leftrightarrow} scc(Y)$,

    where neighboring SCCs are not identical. As $j \in N^{-int}(E) \backslash S$, there is a $scc(X^{i_l})$ such that $scc(X^{i_l})$ is not in $N^{-int}(E)$; otherwise, $j \in re(Y) \subset S$. Moreover, there must be a component collider $scc(X^{i_q})$ between $scc(X^{i_l})$ and Y, because $Y$ is not a descendant of $N^{int}(E)$. Choose $q$ which is the smallest integer such that $scc(X^{i_q}) \subset de(X^{scc(X^{i_l})}) \subset N^{int}(E)$. Then we can find a collider in $scc(X^{i_q})$, so $P$ is $\sigma$-blocked  by $S$.

    \item  $X^j \cdots X^k \rightarrow scc(X^{i_n}) \twoset{-}{\leftrightarrow} \cdots scc(X^{i_1}) \twoset{-}{\leftrightarrow} scc(Y)$,
    
    where $k \notin re(Y)$. If $k = j$, then there must be a $scc(X^{i_l})$ such that $scc(X^{i_l}) \subset N^{int}(E)$; otherwise, $j \in pa(re(Y))$, contradicting the condition. Likewise, a component collider $scc(X^{i_q})$ exists between $scc(X^{i_l})$ and $scc(Y)$, and $scc(X^{i_q}) \subset N^{int}(E)$, which means that $P$ can be $\sigma$-blocked by $S$ because it has a collider. 
    
    When $k \neq j$, then $X^k$ may have three cases: (1) $k \in S$. $X^k$ is a non-collider on $P$ pointing to a different SCC. So $P$ is $\sigma$-blocked by $S$. (2) $k \in N^{int}(E) $. Let $q$ be the smallest integer, provided that $scc(X^{i_q}) \subset de(X^k)$. We know that $scc(X^{i_q}) \neq scc(Y)$ as $Y$ is not in $N^{int}(E)$. $scc(X^{i_q})$ is a component collider in the sequence of components, so there must be a collider in $scc(X^{i_q})$. Thus $P$ can be $\sigma$-blocked by $S$. (3) $k \in N^{-int}(E) \backslash S$. There must be a $scc(X^{i_l})$ such that $scc(X^{i_l}) \subset N^{int}(E)$; otherwise, $k \in pa(re(Y))$. Thus the sequence between $scc(X^{i_l})$ and $scc(Y)$ has a component collider $scc(X^{i_q}) \subset N^{int}(E)$ where a collider exists. $P$ is $\sigma$-blocked by $S$.
    
    \item $X^j \cdots X^k \leftarrow scc(X^{i_n}) \twoset{-}{\leftrightarrow} \cdots scc(X^{i_1}) \twoset{-}{\leftrightarrow} scc(Y)$,
    
    where $k \notin re(Y)$. Note $j$ is likely to be equal to $k$. Let $l \in scc(X^{i_n})$ satisfy that $X^k \leftarrow X^l$.

    If $l \in S$, then $P$ is $\sigma$-blocked by $S$ as $X^l$ is a non-collider on $P$ and points to another SCC.
    
    If $l \notin S$, moreover, if $l \in N^{int}(E)$, there must be a collider between $X^l$ and $X^j$ as $j \notin N^{int}(E)$. Let $X^q$ be the farthest descendant of $X^l$ located between $X^l$ and $X^j$. It is obvious that $q \in N^{int}(E)$. So $P$ is $\sigma$-blocked by $S$ as $X^q$ is a collider on $P$. On the other hand, if $l \in N^{-int}(E) \backslash S$, there will be an $i_q$ such that $i_q \in N^{int}(E)$; otherwise, $l \in re(Y)$ in $\mathcal{G}^{(N^{-int}(E),Y,I)}$. Moreover, there must be a component collider between $scc(X^{i_q})$ and $scc(Y)$. Therefore, $P$ is $\sigma$-blocked by $S$ as there is a collider on $P$.
    
    \item $scc(X^j) \twoset{-}{\leftrightarrow} scc(X^{i_n}) \twoset{-}{\leftrightarrow} \cdots \twoset{-}{\leftrightarrow} scc(X^{i_1}) \leftarrow Y$,
    
    where $i_1 \in ch(Y)$ and adjacent SCCs are not identical. As $j \notin S$, there must be an $i_l$ satisfying $scc(X^{i_l}) \subset N^{int}(E)$. The existence of a component collider between $scc(X^{i_l})$ and $Y$ guarantees a collider on $P$. So $P$ is $\sigma$-blocked by $S$. 
    
    \item $X^j \cdots X^k \rightarrow scc(X^{i_n}) \twoset{-}{\leftrightarrow} \cdots \twoset{-}{\leftrightarrow} scc(X^{i_1}) \leftarrow Y$,
    
    where $k \notin scc(X^{i_n})$. If $j = k$, as $j \notin S$, there must be an $i_l$ such that $scc(X^{i^l}) \subset N^{int}(E)$. Such a SCC can result in a collider on $P$ as in the above discussion, so $P$ can be $\sigma$-blocked by $S$.
    
    If $j \neq k$, there will be three situations: (1) $k \in S$, then $P$ is $\sigma$-blocked by $S$ as $X^k$ is a non-collider on $P$ with an edge pointing to another SCC. (2) $k \in N^{int}(E)$. Under this condition, let $q$ be the smallest integer such that $scc(X^{i_q}) \subset de(X^k)$. $scc(X^{i_q})$ is a component collider in this sequence. Besides, $scc(X^{i_q}) \subset N^{int}(E)$ and $de(scc(X^{i_q})) \subseteq N^{int}(E)$. So $P$ is $\sigma$-blocked by $S$ because there is a collider on $P$ located in $scc(X^{i_q})$. (3) $k \in N^{-int}(E) \backslash S$. Then there must be an $i_q$ such that $scc(X^{i_q}) \subset N^{int}(E)$; otherwise, $k \in pa(re(ch(Y))) \subset S$. Likewise, $P$ is $\sigma$-blocked by $S$ as there will be a collider on $P$.
    
    \item $X^j \cdots X^k \leftarrow scc(X^{i_n}) \twoset{-}{\leftrightarrow} \cdots \twoset{-}{\leftrightarrow} scc(X^{i_1}) \leftarrow Y$,
    
    where $k \notin scc(X^{i_n})$. Let $l \in scc(X^{i_n})$ have $X^k \leftarrow X^l$.
    
    If $l \in S$, $P$ is $\sigma$-blocked by $S$ as $X^l$ is a non-collider on $P$ and has an edge to a different SCC.
    
    If $l \in N^{int}(E)$, there must be a collider between $X^j$ and $X^l$ because $j \in N^{-int}(E)$. Thus $P$ can be $\sigma$-blocked by $S$. 
    
    If $k \in N^{-int}(E) \backslash S$, there will be an $i_q$ such that $scc(X^{i_q}) \subset N^{int}(E)$. Similarly, $P$ can be $\sigma$-blocked by $S$ as there is a collider on $P$.
    
    \end{enumerate}

Then we prove that the stable frontier is intervention-stable. From the construction of $N^{int}(E)$, $scc(Y) \subseteq N^{-int}(E)$ (otherwise, there will be a directed path from $N^{int}(E)$ to $Y$). Thus $scc(Y) \subseteq SF_I(Y,E)$. Secondly, parallel to the proof of Theorem \ref{theorem 4}, let $P$ be a path from $I^l$ to $Y$, $l \in \{1,\cdots,m\}$. Then $P$ has the following shapes.

\begin{enumerate}[(i)]
    \item $I^l \rightarrow \cdots X^j \rightarrow scc(X^{i_n}) \twoset{-}{\leftrightarrow} \cdots \twoset{-}{\leftrightarrow} scc(X^{i_1}) \twoset{-}{\leftrightarrow} scc(Y),$
    
    where $n \geq 0$ and $j \notin ma(X^{i_n})$. Let $scc(X^{i_k})$ be the farthest SCC away from $scc(Y)$ such that $scc(X^{i_1}), \cdots, scc(X^{i_k}) \subseteq N^{-int}(E)$ and $scc(X^{i_{k+1}}) \subseteq N^{int}(E)$ ($k < n$). If $k = n$, $j \in N^{-int}(E)$, then $j \in pa(re(Y))$ in the sub-graph, so $j \in SF_I(Y,E)$. Thus $P$ is $\sigma$-blocked by $SF_I(Y,E)$ as $X^j$ is a non-collider pointing out to another SCC.
    
    Otherwise, $k < n$ and $scc(X^{i_{k+1}}) \subseteq N^{int}(E)$. Then there must be a component collider between $X^j$ and $scc(X^{i_k})$ because the middle part shape of $P$ is either $X^j \rightarrow \cdots \leftarrow scc(X^{i_k})$ or $X^j \rightarrow \cdots \leftrightarrow scc(X^{i_k})$. Let $scc(X^{i_q})$ be the closest component collider to $scc(X^{i_k})$ between $X^j$ and $scc(X^{i_k})$. So $i_q \in de(scc(X^{i_{k+1}}))$, then $scc(X^{i_k}) \cap N^{-int}(E) = \emptyset$ and $de(scc(X^{i_k})) \cap N^{-int}(E) = \emptyset$, which means $P$ is $\sigma$-blocked by $SF_I(Y,E)$ by Proposition \ref{component block}. 
    
    \item  $I^l \rightarrow \cdots X^j \leftarrow scc(X^{i_n}) \twoset{-}{\leftrightarrow} \cdots \twoset{-}{\leftrightarrow} scc(X^{i_1}) \twoset{-}{\leftrightarrow} scc(Y),$ 
    
    where $n \geq 1$ and $j \notin ma(X^{i_n})$. Let $scc(X^{i_k})$ be the farthest SCC away from $scc(Y)$ such that $scc(X^{i_1}), \cdots, scc(X^{i_k}) \subseteq N^{-int}(E)$ and $scc(X^{i_{k+1}}) \subseteq N^{int}(E)$ ($k < n$). If $k = n$, there exists an $t \in scc(X^{i_n})$ such that $X^j \leftarrow X^t$ is a part of $P$. $t \in re(Y)$ in the sub-graph, so $t \in SF_I(Y,E)$. Thus $P$ is $\sigma$-blocked by $SF_I(Y,E)$ as $X^t$ is a non-collider on $P$ pointing out to another SCC.
    
    Otherwise, if $k < n$, then there must be a collider between $I^l$ and $scc(X^{i_k})$ because the middle part shape of $P$ is either $I^l \rightarrow \cdots \leftarrow scc(X^{i_k})$ or $I^l \rightarrow \cdots \leftrightarrow scc(X^{i_k})$. Let $X^q$ is the closest collider to $scc(X^{i_k})$ on $P$ between $I^l$ and $scc(X^{i_k})$, then $q \in de(scc(X^{i_{k+1}}))$. Hence $q \notin N^{-int}(E)$ and $q \notin an(N^{-int}(E))$. $P$ is still $\sigma$-blocked by $SF_I(Y,E)$ because $SF_I(Y,E) \subseteq N^{-int}(E)$ and $X^q$ is a collider.
    
    \item $I^l \rightarrow \cdots X^j \leftarrow scc(Y),$
    
    where $j \notin scc(Y)$ and $j \notin ma(Y)$. And the part of $P$ in $scc(Y)$ is $(X^{i_n}, e_1, X^{i_{n-1}}, \cdots, Y)$ such that $X^j \leftarrow X^{i_n}$. As $i_n \in scc(Y) \subseteq SF_I(Y,E)$ in the sub-graph, $P$ is $\sigma$-blocked by $SF_I(Y,E)$ as $X^{i_n}$ is a non-collider on $P$ which points out to a different SCC.
    
    \item $I^l \rightarrow scc(X^{i_n}) \twoset{-}{\leftrightarrow} \cdots  \twoset{-}{\leftrightarrow} scc(X^{i_1}) \leftarrow Y,$ 
    
    where $i_1 \notin scc(Y)$, $i_1 \in ch(Y)$, and $i_1 \notin ma(Y)$. In this case, there must be an intervened component collider $scc(X^{i_k})$ on this intervened component district which is in the $E$. So $scc(X^{i_k}) \cap N^{-int}(E) = de(scc(X^{i_k})) \cap N^{-int}(E) = \emptyset$, then $P$ is $\sigma$-blocked by $SF_I(Y,E)$ as $SF_I(Y,E) \subseteq N^{-int}(E)$ by Proposition \ref{component block}.
    
    \item $I^l \rightarrow \cdots X^j \rightarrow scc(X^{i_n}) \twoset{-}{\leftrightarrow} \cdots  \twoset{-}{\leftrightarrow} scc(X^{i_1}) \leftarrow Y,$
    
    where $j \notin ma(X^{i_n})$, $j \notin scc(X^{i_n})$, $i_1 \notin scc(Y)$, $i_1 \in ch(Y)$, and $i_1 \notin ma(Y)$. Let $scc(X^{i_k})$ be the farthest SCC away from $Y$ such that $scc(X^{i_1}), \cdots, scc(X^{i_k}) \subseteq N^{-int}(E)$ and $scc(X^{i_{k+1}}) \subseteq N^{int}(E)$ ($k < n$). If $k = n$, $j \in N^{-int}(E)$, then $j \in pa(re(ch(Y)))$ in the sub-graph, so $j \in SF_I(Y,E)$. Thus $P$ is $\sigma$-blocked by $SF_I(Y,E)$ as $X^j$ is a non-collider pointing out to another SCC.
    
    Otherwise, $k < n$ and $scc(X^{i_{k+1}}) \subseteq N^{int}(E)$. Then there must be a component collider between $X^j$ and $scc(X^{i_k})$ because the middle part shape of $P$ is either $X^j \rightarrow \cdots \leftarrow scc(X^{i_k})$ or $X^j \rightarrow \cdots \leftrightarrow scc(X^{i_k})$. Let $scc(X^{i_q})$ be the closest component collider to $scc(X^{i_k})$ between $X^j$ and $scc(X^{i_k})$. So $i_q \in de(scc(X^{i_{k+1}}))$, then $scc(X^{i_k}) \cap N^{-int}(E) = \emptyset$ and $de(scc(X^{i_k})) \cap N^{-int}(E) = \emptyset$, which means $P$ is $\sigma$-blocked by $SF_I(Y,E)$ by Proposition \ref{component block}.

    \item $I^l \rightarrow \cdots X^j \leftarrow scc(X^{i_n}) \twoset{-}{\leftrightarrow} \cdots  \twoset{-}{\leftrightarrow} scc(X^{i_1}) \leftarrow Y,$ 
    
     where $j \notin ma(X^{i_n})$, $j \notin scc(X^{i_n})$, $i_1 \notin scc(Y)$, $i_1 \in ch(Y)$, and $i_1 \notin ma(Y)$. Let $scc(X^{i_k})$ be the farthest SCC away from $Y$ such that $scc(X^{i_1}), \cdots, scc(X^{i_k}) \subseteq N^{-int}(E)$ and $scc(X^{i_{k+1}}) \subseteq N^{int}(E)$ ($k < n$). If $k = n$, there exists an $t \in scc(X^{i_n})$ such that $X^j \leftarrow X^t$ is a part of $P$. $t \in re(ch(Y))$ in the sub-graph, so $t \in SF_I(Y,E)$. Thus $P$ is $\sigma$-blocked by $SF_I(Y,E)$ as $X^t$ is a non-collider on $P$ pointing out to another SCC.
    
    Otherwise, if $k < n$, then there must be a collider on $P$ between $I^l$ and $scc(X^{i_k})$ because the middle part shape of $P$ is either $I^l \rightarrow \cdots \leftarrow scc(X^{i_k})$ or $I^l \rightarrow \cdots \leftrightarrow scc(X^{i_k})$. Let $X^q$ is the closest collider to $scc(X^{i_k})$ on $P$ between $I^l$ and $scc(X^{i_k})$, then $q \in de(scc(X^{i_{k+1}}))$. Hence $q \notin N^{-int}(E)$ and $q \notin an(N^{-int}(E))$. $P$ is still $\sigma$-blocked by $SF_I(Y,E)$ because $SF_I(Y,E) \subseteq N^{-int}(E)$ and $X^q$ is a collider.
\end{enumerate}

From the above-classified discussions on how the path enters $Y$, we know that the stable frontier $SF_I(Y,E)$ for any eligible set of sub-district colliders $E$ is also intervention-stable with respect to all interventions. 
\end{proof}

\begin{example}
Given a latent projected graph over $(X^1,\cdot,X^9,Y, I^1,I^2)$ of Setting \ref{Setting 4}, we can read off the Markov blanket directly from the graph.

\begin{figure}[ht]
\centering
\begin{tikzpicture}[
    >=Stealth,
    thick,
    scale=1.05,
    every node/.style={font=\normalsize},
    obsnode/.style={
        circle,
        draw=black,
        fill=white,
        minimum size=9mm,
        inner sep=0pt
    },
    intnode/.style={
        rectangle,
        draw=black,
        fill=white,
        minimum width=9mm,
        minimum height=9mm,
        inner sep=0pt
    },
    ynode/.style={
        double,
        circle,
        draw=black,
        fill=white,
        minimum size=10mm,
        inner sep=0pt,
        line width=0.8pt
    }
]

    \node[intnode] (i1) at (-4,1.2) {$I^1$};
    \node[obsnode] (x1) at (-2.2,1.2) {$X^1$};
    \node[ynode]   (y)  at (0,1.2) {$Y$};

    \node[obsnode] (x5) at (2.2,2.6) {$X^5$};
    \node[obsnode] (x6) at (2.2,0.1) {$X^6$};

    \node[obsnode] (x2) at (0,-1.0) {$X^2$};
    \node[obsnode] (x3) at (2,-1.0) {$X^3$};
    \node[obsnode] (x4) at (1,-2.5) {$X^4$};

    \node[obsnode] (x7) at (4.8,-1.0) {$X^7$};
    \node[obsnode] (x8) at (6.8,-1.0) {$X^8$};
    \node[obsnode] (x9) at (5.8,-2.5) {$X^9$};
    \node[intnode] (i2) at (7.8,-2.5) {$I^2$};

    \draw[->] (x1) -- (y);
    \draw[->] (y) -- (x2);

    \draw[->,  bend left=18]  (y) to (x5);
    \draw[<->, bend right=18] (y) to (x5);

    \draw[->] (x5) -- (x6);
    \draw[->] (x6) -- (y);

    \draw[->] (x2) -- (x3);
    \draw[->] (x3) -- (x4);
    \draw[->] (x4) -- (x2);

    \draw[<->] (x3) -- (x7);

    \draw[->] (x7) -- (x8);

    \draw[->,  bend left=18]  (x8) to (x9);
    \draw[<->, bend right=18] (x8) to (x9);

    \draw[->] (x9) -- (x7);

    \draw[->] (i1) -- (x1);
    \draw[->] (i2) -- (x9);

\end{tikzpicture}

\caption{Cyclic DMG generated by latent projection with respect to hidden variables. $I^1$ and $I^2$ are interventions. The Markov blanket of $Y$ consists of all $X$s. However, there are two intervened component colliders $scc(X^9)$ and $scc(X^3)$ on the intervened component district $I^2 \rightarrow scc(X^9) \leftrightarrow scc(X^3) \leftarrow Y$. So one possible stable frontier is $\{X^1,X^2,X^3,X^4,X^5,X^6\}$, and the other stable frontier is $\{X^1,X^5,X^6\}$.}
\label{fig:M5}
\end{figure}
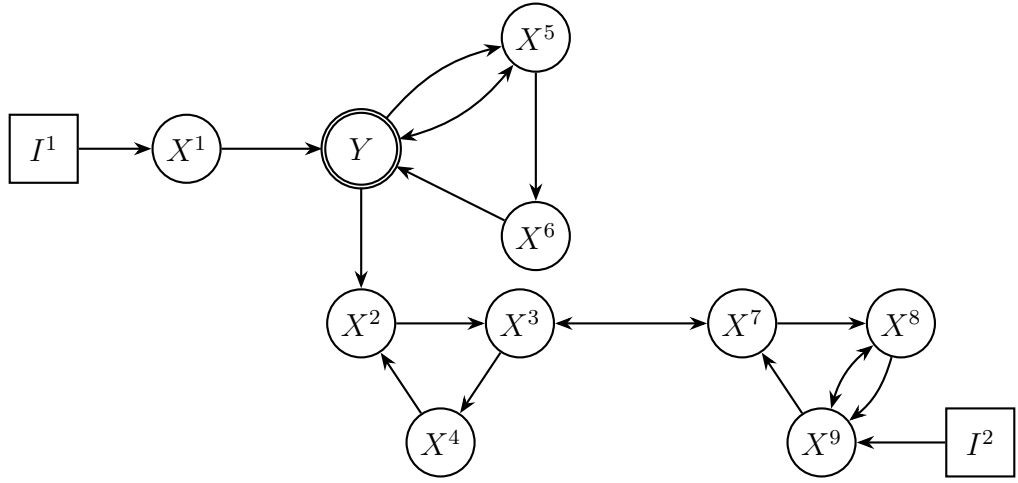
\end{example}

\begin{proposition}
\label{remove sub}
For each intervened component district 
$$
I \rightarrow scc(X^{i_n}) \twoset{-}{\leftrightarrow} scc(X^{i_{n-1}}) \twoset{-}{\leftrightarrow} \cdots \twoset{-}{\leftrightarrow} scc(X^{i_1}) \leftarrow Y,$$
if there is a condition set $S$ such that $X^S$ can $\sigma$-block all paths between $I$ and $Y$, then there exists an intervened component collider $scc(X^{i_k})$ such that elements in $scc(X^{i_k})$ and $de(scc(X^{i_k}))$ are not in the condition set.
\end{proposition}

\begin{proof}
Firstly, there are paths with the shape as 
$$
I \rightarrow scc(X^{i_n}) \leftrightarrow scc(X^{i_{n-1}}) \leftrightarrow \cdots \leftrightarrow scc(X^{i_1}) \leftarrow Y.
$$
For these paths, no non-endpoint non-colliders which point out to another SCC exist on them. Thus since $X^S$ can $\sigma$-block all paths between $I$ and $Y$, there is an $i_k$ such that $scc(X^{i_k}) \cap S = de(scc(X^{i_k})) \cap S = \emptyset$. Let $scc(X^{i_k})$ be the farthest SCC away from $Y$ satisfying the above condition. If $scc(X^{i_k})$ is also a component collider and there is no directed path from $scc(X^{i_k})$ to $Y$, then $scc(X^{i_k})$ is an intervened component collider.

However, if there is a directed path from $scc(X^{i_k})$ to $Y$ as $X^j \rightarrow \cdots \rightarrow Y$ where $j \in scc(X^{i_k})$, then we consider paths having a shape as 
$$
I \rightarrow scc(X^{i_n}) \leftrightarrow \cdots \leftrightarrow scc(X^{i_k}) \rightarrow \cdots \rightarrow Y.
$$
Since $scc(X^{i_k})$ is the farthest SCC such that elements in $scc(X^{i_k})$ and $de(scc(X^{i_k}))$ are not in $S$, nodes in $scc(X^{i_n}), \cdots, scc(X^{i_{k+1}})$ can not act the collider on $P$ such that itself and its descendants are not in $S$. At the same time, nodes on $X^j \rightarrow \cdots \rightarrow Y$ are not in $S$. So these paths can not be $\sigma$-blocked by $S$, contradicting the condition. Thus, there is no directed path from $scc(X^{i_k})$ to $Y$.

Then we look at whether $scc(X^{i_k})$ is a component collider. 

\begin{enumerate}[(i)]
    \item $I \rightarrow scc(X^{i_n}) \twoset{-}{\leftrightarrow} \cdots \twoset{\leftarrow}{\leftrightarrow} scc(X^{i_k}) \twoset{-}{\leftrightarrow} \cdots \twoset{-}{\leftrightarrow} scc(X^{i_1}) \leftarrow Y.$
    
    If there is a leftward arrow starting from $scc(X^{i_k})$, there must be a component collider between $I$ and $scc(X^{i_k})$. Let $scc(X^{i_q})$ is the closest one to $scc(X^{i_k})$, then $i_q \in de(scc(X^{i_k}))$ and $de(X^{i_q}) \in de(scc(X^{i_k}))$. So $scc(X^{i_q}) \cap S = de(scc(X^{i_q})) \cap S = \emptyset$, which is impossible as $scc(X^{i_k})$ is the farthest SCC satisfying such property.
    
    \item $I \rightarrow scc(X^{i_n}) \twoset{-}{\leftrightarrow} \cdots \twoset{-}{\leftrightarrow} scc(X^{i_k}) \twoset{\rightarrow}{\leftrightarrow} \cdots \twoset{-}{\leftrightarrow} scc(X^{i_1}) \leftarrow Y.$
    
    If there is a rightward arrow starting from $scc(X^{i_k})$, there must be a component collider between $Y$ and $scc(X^{i_k})$. Let $scc(X^{i_q})$ is the closest one to $scc(X^{i_k})$, then $i_q \in de(scc(X^{i_k}))$ and $de(X^{i_q}) \subset de(scc(X^{i_k}))$. Besides, if there is a directed path from $scc(X^{i_q})$ to $Y$, it will generate a directed path from $scc(X^{i_k})$ to $Y$ as $i_q \in de(scc(X^{i_k}))$, which is impossible as we state before. Thus $scc(X^{i_q})$ is an intervened component collider.
\end{enumerate}

Therefore, if $I \indep_\sigma Y | X^S$, then there must be an intervened component collider such that itself and its descendant set have no intersection with the condition set $S$.

\end{proof}

By Proposition \ref{remove sub}, for each intervened component district, an intervention-stable set should remove elements in one intervened component collider and its descendant set, in order to $\sigma$-separate all interventions and $Y$. Intuitively, selecting one intervened component collider per intervened component district when constructing an eligible set of components $E$, as well as choosing the one furthest from Y, may make the intervention-stable set $N^{-int}(E^\prime)$ contain more variables.

\begin{corollary}
Given a graphical model under Setting \ref{Setting 4}, let $E^\prime$ be the eligible set of components consisting of the farthest intervened component collider away from $Y$ on each intervened component district. $N^{int}(E^\prime)$ is the intervention set given by $E^\prime$. And
$$
N^{-int}(E^\prime) = \{1,\cdots,p\} \backslash N^{int}(E^\prime).
$$
Then $N^{-int}(E^\prime)$ is an intervention-stable set with respect to all the interventions. 
\end{corollary}

\begin{proof}
The proof is similar to the proof of Theorem \ref{theorem 4}.
\end{proof}

Under more assumptions, we can see that the stable frontiers determined by different eligible set of components are the same. Therefore, we define the unique stable frontier as stable blanket of $Y$ under interventions.

\begin{theorem}
\label{theorem 5}
Given a graphical model under Setting \ref{Setting 4}, if one of the two assumptions is satisfied, 
\begin{enumerate}[(i)]
    \item there is no intervened component district, or
    \item for every pair of disjoint intervened component colliders on one intervened component district, there is an intervened component district with only one intervened component collider such that the intervened component collider has descendants in both intervened component colliders of the pair.
\end{enumerate}

\noindent then the $N^{int}(E)$ keeps the same among different eligible sets of components $E$. Thus the stable frontier $CB_I(Y, E)$ exists and is unique. We call it the stable blanket of $Y$, denoted by $SB_I(Y)$.
\end{theorem}

\begin{proof}
Firstly, if there is no intervened component district in the graphical model, then the $N^{int}$(E) is empty, thus the stable frontier is exact the Markov blanket of $Y$. 

Otherwise, given two different eligible sets of components $E^1$ and $E^2$, there must be an intervened component district where $E^1$ and $E^2$ have different choices of intervened component colliders. Let the intervened component colliders of $E^1$ be $(s^1,\cdots,s^n)$ and the intervened component colliders of $E^2$ be $(c^1,\cdots,c^m)$. For every $(i, j)$, $i=1,\cdots,n$ and $j=1,\cdots,m$, there is a common intervened component collider $l^{i,j}$ collider in both $E^1$ and $E^2$, because $l^{i,j}$ is the only intervened component collider on an intervened component district. According to the assumption, $de(l^{i,j}) \cap s^i \neq \emptyset$ and $de(l^{i,j}) \cap c^j \neq \emptyset$, so $X^{c^j} \subseteq N^{-int}(E^1)$ and $X^{s^i} \subseteq N^{-int}(E^2)$. Besides, it has $X^{c^j} \subseteq N^{-int}(E^2)$ and $X^{s^i} \subseteq N^{-int}(E^1)$. Thus, $N^{-int}(E^1)$ or $N^{-int}(E^2)$ is not affected by different eligible sets of components. 
\end{proof}

\section{Interventions on Districts}

In Setting \ref{Setting 2}, we assume that there is no interventions on $dis(Y)$. By Proposition \ref{sub-district collider}, the aim of the assumption is to guarantee that there is a sub-district collider on each intervened sub-district. However, in some cases, although there are interventions on $dis(Y)$, it is still possible to construct a set of predictors which can explain variations of interventions. By Proposition \ref{remove sub hidden}, there must be a sub-district collider on each intervened sub-district. We will start from this necessary condition to prove it is sufficient. Moreover, we can loosen the assumptions such that it can contains case with interventions on $dis(Y)$.

\begin{definition}[Intervened District]
Assume there is an intervention $I$ acting on the district of of Y. Then an intervened district is a series of nodes $(i_1, \cdots, i_n)$ which can be connected as 
\begin{equation}
    I \rightarrow X^{i_n} \twoset{-}{\leftrightarrow} X^{i_{n-1}} \twoset{-}{\leftrightarrow} \cdots \twoset{-}{\leftrightarrow} X^{i_1} \twoset{-}{\leftrightarrow} Y \label{int dis}
\end{equation}
where $i_1,\cdots,i_n \in \{1,\cdots,p\}$.

\end{definition}

Similar to Proposition \ref{remove sub hidden}, there is a new proposition.

\begin{definition}[District Collider]
A district collider on an intervention district is a node if satisfying that
\begin{enumerate}[(i)]
    \item edges preceding and succeeding it in the intervention district have an arrowhead into it, and
    \item there is no directed path from it to $Y$.
\end{enumerate}
\end{definition}

\begin{proposition}
For each intervened district 
$$
I \rightarrow X^{i_n} \twoset{-}{\leftrightarrow} X^{i_{n-1}} \twoset{-}{\leftrightarrow} \cdots \twoset{-}{\leftrightarrow} X^{i_1} \twoset{-}{\leftrightarrow} Y,
$$
if a condition set $X^S$ can $m$-block all paths between $I$ and $Y$, then there exists a district collider such that itself and its descendants are not in the condition set.
\end{proposition}

\begin{proof}
$X^S$ can $m$-block all paths between $I$ and $Y$, so it can $m$-block
$$
I \rightarrow X^{i_n} \leftrightarrow X^{i_{n-1}} \leftrightarrow \cdots \leftrightarrow X^{i_1} \leftrightarrow Y.
$$
There must be a collider on the path such that itself and its descendants are not in $S$. Let $X^{i_k}$ is the farthest such collider away from $Y$. Then there is no directed path from $X^{i_k}$ to $Y$. Otherwise, the path
$$
I \rightarrow X^{i_n} \leftrightarrow \cdots \leftrightarrow X^{i_k} \rightarrow \cdots \rightarrow Y 
$$
can not be $m$-blocked by $X^S$. Moreover, $X^{i_k}$ can not have an directed edge pointing to $X^{i_{k+1}}$ (otherwise, $X^{i_{k+1}}$ and its descendants are not in $S$, contradicting that $X^{i_k}$ is the farthest node). If $X^{i_k}$ have an directed edge pointing to $X^{i_{k-1}}$, then there must be a district collider between $X^{i_k}$ and $Y$ (otherwise, there is a directed path from $X^{i_k}$ to $Y$). Let $X^{i_q}$ be the closest district collider to $X^{i_k}$; then $i_q \in de(X^{i_k})$, which means there is no directed path from $X^{i_q}$ to $Y$. So $X^{i_q}$ is a district collider on the intervened district.
\end{proof}

Conversely, given a district collider on the intervened district, then all paths having shape as \eqref{int dis} can be $m$-blocked by a condition set if the district collider and its descendants are not in the condition set.

\begin{setting}
\label{Setting 5}
So rather than assume no interventions act on $dis(Y)$, we replace the assumption in Setting \ref{Setting 2} by the assumption that for each intervened district and intervened sub-district, there is at least one district collider and sub-district collider, respectively. We call this as Setting 6.1.
\end{setting}

In Setting \ref{Setting 5}, we can define overall set of colliders.

\begin{definition}[Overall Set of Colliders]
Let $s^1,\cdots,s^m$ be all intervened sub-districts and $d^1,\cdots,d^n$ be all intervened districts, an overall set of colliders $O$ is a set such that 
$$
O = \{c^{1,1}, \cdots c^{1,i_1}, \cdots, c^{m,1}, \cdots c^{m,i_m}, b^{1,1}, \cdots b^{1,j_1}, \cdots, b^{n,1}, \cdots c^{n,i_n} \}
$$
where $c^{k,\cdot}$ is a sub-district collider on $S^k$ and $b^{h,\cdot}$ is a district collider on $d^h$, $k =1,\cdots,m,\ h = 1,\cdots,n$.
\end{definition}

We overuse $N^{int}$ to denote the set consists of an overall set of colliders and their descendants. And we say its complementary set is intervention-stable.

\begin{theorem}
Given a graphical model in Setting \ref{Setting 5}, let $O$ be an overall set of colliders. $N^{int}(O)$ is the intervention set consisting of $O$ and $de(X^O)$. And
$$
N^{-int}(O) = \{1,\cdots,p\} \backslash N^{int}(O).
$$
Then $N^{-int}(O)$ is an intervention-stable set with respect to all the interventions. 
\end{theorem}

\begin{proof}
We can also divide paths by their shapes between interventions $I^l$ and $Y$. The only difference is that 
$$
I^l \rightarrow X^{i_n} \twoset{-}{\leftrightarrow} X^{i_{n-1}} \twoset{-}{\leftrightarrow} \cdots \twoset{-}{\leftrightarrow} X^{i_1} \twoset{-}{\leftrightarrow} Y
$$
can appear in Setting \ref{Setting 5}. Since there is a district collider and its descendants are not in $N^{-int}(O)$, the paths with this shape can also be $m$-blocked by $X^{N^{-int}(O)}$. Therefore, $N^{-int}(O)$ is intervention-stable with respect to all interventions.
\end{proof}

Thereafter, we can define the stable frontier and the stable blanket in Setting \ref{Setting 5}, likewise Section 3.

\section{Discussion}
In this paper, for three settings: containing only hidden variables, containing only cycles, and containing both hidden variables and cycles, we find Markov blankets and stable blankets from different graphs, respectively. The Markov blanket is optimal for the set of predictors when there is only one environment, while the stable blanket is optimal for multiple environments given that it can be generalized to unseen environments. When there are enough interventions, the Markov blanket converges to the stable blanket. At the same time, we discuss the assumptions about the interventions and find a sufficient assumption that allows the existence of the stable blanket to be guaranteed. However, we have not yet discussed the condition for faithfulness, and how to identify the optimal set of predictors from the data and give a causal explanation would be a direction worth exploring in depth.

\section*{Acknowledgements}
I am grateful to Niklas Pfister for his supervision and many helpful discussions on this work.

\end{document}